\documentclass[sigconf]{acmart}

\AtBeginDocument{%
  \providecommand\BibTeX{{%
    \normalfont B\kern-0.5em{\scshape i\kern-0.25em b}\kern-0.8em\TeX}}}

\usepackage{microtype}
\usepackage{graphicx}
\usepackage{hyperref}
\usepackage{balance}

\usepackage{enumitem}

\usepackage{url}  %Required
\usepackage{graphicx}  %Required

\usepackage{algorithm}
\usepackage{algorithmic}
\usepackage{bm}

\usepackage[utf8]{inputenc} % allow utf-8 input
\usepackage{amsfonts}       % blackboard math symbols
\usepackage{nicefrac}       % compact symbols for 1/2, etc.

\usepackage{amsmath}
\usepackage{amsthm}
\usepackage{subcaption}
\usepackage{flushend}

\newtheorem{theorem}{Theorem}[section]
\newtheorem{lemma}[theorem]{Lemma}
\newtheorem{remark}{Remark}

\allowdisplaybreaks

\copyrightyear{2020}
\acmYear{2020}
\setcopyright{usgovmixed}\acmConference[KDD '20]{Proceedings of the 26th ACM SIGKDD Conference on Knowledge Discovery and Data Mining}{August 23--27, 2020}{Virtual Event, CA, USA}
\acmBooktitle{Proceedings of the 26th ACM SIGKDD Conference on Knowledge Discovery and Data Mining (KDD '20), August 23--27, 2020, Virtual Event, CA, USA}
\acmPrice{15.00}
\acmDOI{10.1145/3394486.3403243}
\acmISBN{978-1-4503-7998-4/20/08}

\begin{document}

\fancyhead{} 

\title{Combinatorial Black-Box Optimization with Expert Advice}


\author{Hamid Dadkhahi}
%\authornote{Dr.~Trovato insisted his name be first.}
\affiliation{%
  \institution{IBM Research}
%   \streetaddress{NY}
}
\email{hdadkhahi@ibm.com}

\author{Karthikeyan Shanmugam}
\affiliation{%
  \institution{IBM Research}
%   \streetaddress{NY}
}
\email{karthikeyan.shanmugam2@ibm.com}

\author{Jesus Rios}
\affiliation{%
  \institution{IBM Research}
%   \streetaddress{NY}
}
\email{jriosal@us.ibm.com}

\author{Payel Das}
\affiliation{%
  \institution{IBM Research}
%   \streetaddress{NY}
}
\email{daspa@us.ibm.com}

\author{Samuel C. Hoffman}
\affiliation{%
  \institution{IBM Research}
%   \streetaddress{NY}
}
\email{shoffman@ibm.com}

\author{Troy David Loeffler}
\affiliation{%
  \institution{Argonne National Laboratory}
%   \streetaddress{NY}
}
\email{tloeffler@anl.gov}

\author{Subramanian\:Sankaranarayanan}
\affiliation{%
  \institution{Argonne National Laboratory}
  \institution{University of Illinois at Chicago}
%   \streetaddress{NY}
}
\email{skrssank@uic.edu}


\renewcommand{\shortauthors}{Dadkhahi, Shanmugam and Rios, et al.}



\begin{abstract}

We consider the problem of black-box function optimization over the Boolean hypercube. Despite the vast literature on black-box function optimization over continuous domains, not much attention has been paid to learning models for optimization over combinatorial domains until recently. However, the computational complexity of the recently devised algorithms are prohibitive even for moderate numbers of variables; drawing one sample using the existing algorithms is more expensive than a function evaluation for many black-box functions of interest. To address this problem, we propose a computationally efficient model learning algorithm based on multilinear polynomials and exponential weight updates. In the proposed algorithm, we alternate between simulated annealing with respect to the current polynomial representation and updating the weights using monomial experts' advice. Numerical experiments on various datasets in both unconstrained and sum-constrained Boolean optimization indicate the competitive performance of the proposed algorithm, while improving the computational time up to several orders of magnitude compared to state-of-the-art algorithms in the literature.

\end{abstract}

\maketitle

\section{Introduction}

% Motivation for Combinatorial Black-Box Optimization

% Shortcomings of existing algorithms
Combinatorial optimization (CO) problems arise in numerous application domains including machine learning, engineering, economics, transport, healthcare, and natural and social sciences \cite{wolsey1999integer}. Broadly speaking, such CO problems involve optimizing an explicit function over a constraint set on a discrete domain. A number of important problems in this class are NP-hard and there is a vast literature on approximating them in polynomial time \cite{williamson2011design}. In this work, we focus on black-box combinatorial optimization where we seek to minimize a function defined on the Boolean domain (a combinatorial domain) through acquiring noisy/perfect function evaluations from a black-box oracle. 
%Examples include structural optimization, scientific discovery, and food safety control. These optimization problems are NP-hard and may have constraints. CO problems  involve black-box (and highly non-linear) objectives, and  expensive (and noisy) evaluations. For these reasons, gradient-based optimizers as well as meta-heuristic methods with low sample efficiency, such as local search \cite{Kirkpatrick1983, SA} or evolutionary search \cite{Schafer2013}, are not suitable.  

%Recent trend shows enhanced interest in developing learning-based approaches for CO problems.  Deep learning, reinforcement learning, or a combination thereof, has been used \cite{vinyals2015pointer, silver2016mastering, silver2017mastering,  li2018combinatorial, khalil2017learning}.  

There exists a vast literature on black-box function optimization when it comes to functions over the continuous domains. Bayesian Optimization (BO) \cite{movckus1975bayesian} is a well-established paradigm for optimizing costly-to-evaluate black-box objective functions $f$ with noisy evaluations. The latter paradigm consists of approximating $f$ using a probabilistic function model, often called a \textit{surrogate model}, and utilizing an \textit{acquisition function} along with the surrogate model to draw samples \cite{jones1998efficient}. Some common acquisition functions are Expected Improvement, Probability of Improvement, Thompson Sampling and Upper Confidence Bounds \cite{srinivas2009gaussian,thompson1933likelihood,auer2002using,mockus1994application}. %The majority  of BO literature focuses on objectives with purely continuous domains \cite{BO,hebbal2019bayesian,djolonga2013high,wang2013bayesian}. A subset of the literature is focused on applying BO methods to hyperparameter optimization problems  in machine learning, where the goal is to model the validation error as a function of the hyperparameters in order to find the hyperparameter configuration that obtains the minimum validation error \cite{bergstra2011algorithms}. 
%However, in hyperparameter optimization, the focus is on total training time and not the total number of noisy evaluations of hyperparameters. Therefore, bandit-based and tree search methods that focus on resource allocation have been developed \cite{li2017hyperband,kandasamy2017multi,sen2018multi}. In our work, we exclusively focus on the number of evaluations needed; in addition, the noise level present in any evaluation cannot be controlled and therefore its computational complexity cannot be reduced.

Only recently, generic black-box optimization algorithms, such as BOCS \cite{BOCS} and COMBO \cite{COMBO}, have been proposed for combinatorial domains. However, learning the surrogate model followed by drawing a sample using either BOCS or COMBO, even for a moderate number of variables, is more expensive than an oracle evaluation for many black-box functions of interest. For larger numbers of variables, it is essentially impractical to use BOCS and COMBO, as it takes a significant amount of time to determine the next sample to evaluate.

In this work, we introduce an efficient black box optimization algorithm, that uses a multi-linear polynomial of bounded degree as the surrogate model and sequentially updates this model using exponential weight updates while treating each monomial as an expert. At each step, the acquisition function is a version of simulated annealing applied to the current multilinear polynomial representation given by the surrogate model.  Numerical experiments on various datasets in both unconstrained and sum-constrained Boolean optimization problems indicate the competitive performance of the proposed algorithm, while improving the computational time up to several orders of magnitude compared to state-of-the-art algorithms in the literature.

% Summary of the proposed algorithm and its advantages
\subsection{Contributions}
We summarize our main contributions as follows:
\begin{enumerate}[leftmargin=5mm, itemsep=-1pt]
\item[1.] We propose a novel and computationally efficient algorithm for black-box function optimization over the Boolean hypercube. Our algorithm, Combinatorial Optimization with Monomial Experts (\texttt{COMEX}), comprises a pool of monomial experts forming an approximate multilinear polynomial representation for the black-box function. At any time step, the coefficients of the monomial experts are refreshed via an exponential weight update rule.

\item[2.] The proposed method uses a version of simulated annealing applied to the current polynomial representation to produce new candidate points for black-box function evaluations.

\item[3.] We present theoretical insights on the sequential improvements in the proposed surrogate model as a result of exponential weight updates. Furthermore, we offer theoretical results proving that samples drawn under an exponential acquisition function model lead to sequential improvements in the surrogate model under some technical conditions.

\item[4.] We evaluate the performance of the \texttt{COMEX} algorithm, together with recently developed state-of-the-art BO methods for the combinatorial domain as well as popular heuristic-based baseline methods, on a variety of benchmark problems of different dimensionality. The CO problems  investigated in this study are  sparsification of Ising models, noisy $n$-queens, food contamination control, and optimal arrangement of point defects in 2D nanomaterials. 

\item[5.] \texttt{COMEX} performs competitively on all benchmark problems of low dimensionality, while improving the computational time up to several orders of magnitude. On problems of higher dimensionality, \texttt{COMEX} outperforms all baseline and state-of-the-art BO methods in terms of finding a minimum within a finite time budget. 
\end{enumerate}

%1. We  propose a novel and efficient algorithm for black-box function optimization in discrete space. Our algorithm, Combinatorial Optimization with Monomial Experts (CO-MEX),  casts the problem of optimization for black-box functions over the boolean hypercube. The surrogate model consists of the  multilinear polynomial representation of the boolean function. 

%2. The proposed method uses Monte Carlo sampling with respect to the current polynomial representation to draw samples and  exponential weight updating to update the current estimate of  our surrogate model. 

%3. We evaluate the performance of the CO-MEX algorithm together with recently developed state-of-the-art BO methods for discrete domain and popular heuristic-based baseline methods on a variety of benchmark problems of different dimensionality $n$, including Ising model sparsification, noisy N-queens, and food contamination problem. CO-MEX performs competitively on all benchmark problems of low $n$, while improving the computational time up to a factor of 100. On problems of higher dimensionality, CO-MEX outperforms all baseline and state-of-the-art BO methods in terms of finding a  minimum  within a fixed  time budget. 

\section{Related Work}

% Discrete Optimization literature
The existing algorithms in discrete optimization literature, which are capable of handling black-box functions, are not particularly sample efficient; in many applications, a large evaluation budget is required for such algorithms to converge to functions' optima. In addition, they are not necessarily guaranteed to find the global optima. The most popular algorithms in this category include local search \cite{Kirkpatrick1983, SA} and evolutionary search, such as particle search \cite{Schafer2013}. 

%However, those methods are not designed to be sample efficient, often need extensive manual tuning and domain expertise, and may not converge to a global optimum. Mathematical programming, such as linear, convex, and mixed-integer programming, are techniques to solve discrete optimization problems; however, those cannot be applied to black-box setting. 

% Continuous domains:
\textbf{Bayesian Optimization}:
The majority of work on black-box function optimization targets continuous domains. In particular, algorithms based on Bayesian Optimization \cite{BO} have attracted a lot of attention in the literature. Many popular BO methods are built on top of Gaussian Processes (GPs), which rely on the smoothness defined by a kernel to model uncertainty \cite{srinivas2009gaussian,thompson1933likelihood,mockus1994application}. As such, they are best suited for continuous spaces \cite{BO,hebbal2019bayesian,djolonga2013high,wang2013bayesian}. The only exceptions are the recently introduced algorithm BOCS \cite{BOCS} and COMBO \cite{COMBO}.

\textbf{Hyperparameter Optimization}:
Bayesian Optimization methods have been adapted to hyperparameter optimization \cite{bergstra2011algorithms}. Here, one seeks to find the best hyperparameter configuration that minimizes the validation loss after training a model with that configuration. In this adaptation of BO methods, the goal is to select the next hyperparameter configuration given the function outputs in the previous iterations. However, in hyperparameter optimization, the focus is on the total training time and not the total number of noisy evaluations of hyperparameters. Therefore, bandit-based and tree search methods which focus on resource allocation have been developed \cite{li2017hyperband,kandasamy2017multi,sen2018multi}. In our work, the main cost criterion is the number of function evaluations rather than other resources which can be controlled.
%Bayesian Optimization is an attractive choice for optimizing  black-box functions.  Majority of  BO work targets continuous domains \cite{BO, srinivas2009gaussian, wangICML2017b, hennig2012entropy} and extension to  problems with unknown constrains \cite{gardner2014bayesian, hernandez2015predictive}. Many popular BO methods are built on top of Gaussian Processes (GPs), which rely on the smoothness defined by a kernel to model uncertainty. As such, they are best suited for continuous spaces. The only exceptions are the recently introduced algorithm BOCS \cite{BOCS} and COMBO \cite{COMBO}.

% BOCS
\textbf{Black-Box Combinatorial Optimization}: Similar to our proposed approach, the BOCS algorithm \cite{BOCS} employs a sparse monomial representation to model the interactions among different variables. However, the latter utilizes sparse Bayesian linear regression with a heavy-tailed horse-shoe prior to learn the coefficients of the model, which as we will discuss in the sequel, is computationally costly. A Gibbs sampler is then used in order to draw a sample from the posterior over the monomial coefficients. When the monomials are restricted at the order of two, the problem of minimizing the acquisition function is posed as a second order program which is solved via semidefinite programming. Alternatively, simulated annealing is advocated for higher order monomial models or so as to speed up the computation for the case of order-two monomials. 

% COMBO
More recently, the COMBO algorithm \cite{COMBO} was introduced to address the impediments of BOCS. A one-subgraph-per-variable model is utilized for various combinatorial choices of the variables; the collection of such subgraphs is then joined via graph Cartesian product to construct a combinatorial graph to model different combinations of variables. The Graph Fourier Transform (GFT) over the formed combinatorial graph is used to gauge the smoothness of the black-box function. A GP with a variant of the diffusion kernel, referred to as automatic relevance determination diffusion kernel, is proposed for which GFT can be carried out in a computationally tractable fashion. The proposed GP is capable of accounting for arbitrarily high orders of interactions among variables.

% High computational cost of BOCS/COMBO
The computational complexity of surrogate-model learning in BOCS at time step $t$ is in $\mathcal{O}(t^2 \cdot d^{2 \, m})$, where $d$ is the number of variables and $m$ is the maximum monomial order in the model. This complexity is associated with the cost of sampling parameters from the posterior distribution. Hence, the complexity of BOCS grows quadratically in the number of evaluations, which particularly makes it unappealing for larger numbers of variables.
On the other hand, despite the fact that the GFT utilized by COMBO is shown to be run in linear time with respect to the number of variables for the Boolean case, the overall computational complexity of the algorithm remains prohibitively high. More precisely, the overall computational complexity of learning the surrogate model in COMBO is in $\mathcal{O}(\max\{t^3, d^2\})$. The $\mathcal{O}(t^3)$ complexity is associated with marginal likelihood computation for the GP, whereas the $O(d^2)$ term stems from the slice sampling utilized for fitting the parameters of the surrogate model. 
Therefore, both BOCS and COMBO incorporate model learning methods which grow polynomially in the number of function evaluations. This particularly hinders the applicability of the aforementioned algorithms for problems with moderate numbers of variables, since a larger number of function evaluations is required due to the curse of dimensionality.

% Therefore, for large datasets (i.e, large $d$) and for early stages the $\mathcal{O}(d^2)$ term dominates, while for smaller datasets and later stages $\mathcal{O}(t^3)$ dominates.

% Learning real-valued Boolean functions
%\cite{Sahand} used sparse regression model to learn Boolean functions. \cite{hazan2018} used a sparse monomial model for hyper-parameter optimization.

% Overview of Hedge Algorithm (exponential weight update Algorithm)
\textbf{Prediction with Expert Advice}:
In the framework of prediction with expert advice \cite{prediction_ebook, Vovk1998, hedge}, at each time step $t$ the forecaster receives an instance from a fixed domain. The forecaster is given access to a set of $p$ experts, and is required to produce a distribution $w^t$ over such experts in each time step $t$. Each expert $i$ then incurs a loss $\ell_i^t$, which contributes to the mixture loss of the forecaster given by $\ell^t = \sum_i w_i^t \ell_i^t$. In general, there are no restrictions on the distribution of expert losses. The Hedge algorithm \cite{Vovk1998, hedge} is perhaps the most popular approach to address this problem via an exponential weight update rule given by $w_i^t = w_i^{t-1} \exp(- \eta_t \,\ell_i^t)$, where $\eta$ is a learning rate.

% Overview of sparse online linear regression 
The prediction with expert advice paradigm has been tailored to the problem of sparse online linear prediction from individual sequences. In particular, the $\texttt{EG}^{\pm}$ algorithm \cite{KIVINEN1997} uses an exponential weight update rule to formulate an online linear regression algorithm which performs comparably to the best linear predictor under sparsity assumptions. The adaptive $\texttt{EG}^{\pm}$ algorithm \cite{adaptive_EG} further proposes a parameter-free version of $\texttt{EG}^{\pm}$ where the learning rate $\eta_t$ is updated in an adaptive fashion, and is a decreasing function of time step $t$.

\section{Method}
\label{sec:algorithm}

\subsection{Notations}

Sets are shown with calligraphic letters. For a set $\mathcal{C}$, $|\mathcal{C}|$ designates its cardinality. Matrices are indicated with uppercase letters; vectors and scalars are indicated with lowercase letters. Let $[n] = \{1,2,\ldots,n\}$. For a matrix $A \in \mathbb{R}^{n \times m}$, $A_{i, j}$ designates the element in row $i$ and column $j$. For a vector $x$, $x_i$ indicates its $i$-th entry. We use $\|x\|_p$ to denote the $\ell_p$ norm of a vector $x \in \mathbb{R}^n$, and denote the inner product by $\langle \cdot, \cdot \rangle$.

\subsection{Problem Statement}

We consider the problem of minimizing a black-box function over the Boolean hypercube. The black-box functions of interest are intrinsically expensive to evaluate, potentially noisy, %not amenable to gradient-based optimization approaches, 
% and non-convex, in general. 
% and in general have non-optimal local minima.
and for which in general there is no trivial means to find the minimum.

More precisely, given a subset $\mathcal{C}$ of the Boolean hypercube $\mathcal{X} = \{-1, 1\}^d$, the objective is to find 
\begin{equation}
	x^* = \arg\min_{x \in \mathcal{C}} f(x)
\end{equation}
where $f$ is a real-valued Boolean function $f(x): \mathcal{X} \mapsto \mathbb{R}$. Exhaustive search requires $|\mathcal{C}|$ function evaluations; however, since evaluating the black-box function $f$ is expensive, we are interested in finding $x^*$ (or an approximation of it) in as few function evaluations as possible. In this problem, the performance of any algorithm is measured in terms of \textit{simple regret}, which is the difference between the best evaluation seen until time $t$ and the minimum function value $f(x^*)$: 
\begin{equation}
	R_t = \min_{i \in [t]} |f(x_i) - f(x^*)|.
\end{equation}

Two particularly important instances of such combinatorial structures are $(i)$ \textit{unconstrained optimization problems} where $\mathcal{C}$ includes the entire Boolean hypercube $\mathcal{X}$ where $|\mathcal{C}| = |\mathcal{X}| = 2^d$, and $(ii)$ \textit{optimization problems with a sum constraint} where $\mathcal{C}_n$ corresponds with the $n$-subsets of $[d]$ such that $\sum_i I(x_i = 1) = n$, where $I(.)$ is the indicator function. In the latter problem, we have $|\mathcal{C}_n| = \binom{d}{n}$.

We note that \textit{anytime} algorithms are particularly desirable for this problem for the following reasons: (1)  in many applications the evaluation budget is not known in advance, and (2) the algorithm is run until certain stopping criteria are met. One such stopping criteria is the finite time budget, which is measured as the total computational time required for the algorithm to produce samples to be evaluated by the black-box function, plus the evaluation time consumed by the black-box function of interest. 
% {\color{red} KAR: This comment below is not clear.} The latter budget is particularly relevant in problems where the dimensionality $d$ is of moderate magnitude.

\subsection{Surrogate Model}

In this work, we pursue the framework of using a surrogate model to approximate the black box function along with an acquisition function applied to this surrogate model. At each time step $t$, the surrogate model provides an estimate for the black-box function using the observations $\{(x_i, f(x_i)): i \in [t]\}$ acquired so far. Having been equipped with the new estimate model, the acquisition function selects a new candidate point $x_{t}$ for evaluation. The black-box function then returns the evaluation $f(x_{t})$ for the latter data point. This process is repeated until a stopping criterion, such as an evaluation budget or a time budget, is met.

Any real-valued Boolean function can be uniquely expressed by its \textit{multilinear polynomial} representation \cite{Boolean}:
\begin{equation}
	f(x) = \sum_{\mathcal{I} \subseteq [d]} \alpha^{*}_{\mathcal{I}} \psi_{\mathcal{I}}(x) 
\end{equation}
which is referred to as the \textit{Fourier} expansion of $f$, the real number $\alpha^{*}_{\mathcal{I}}$ is called the Fourier coefficient of $f$ on $\mathcal{I}$, and $\psi_{\mathcal{I}}(x) = \Pi_{i \in \mathcal{I}} x_i$ are monomials of order $|\mathcal{I}|$. The generality of Fourier expansions and the monomials' capability to capture interactions among different variables, make this representation particularly attractive for problems over the Boolean hypercube. 
In addition, in many applications of interest monomials of orders up to $m << d$ are sufficient to capture interactions among the variables, reducing the number of Fourier coefficients from $2^d$ to $p = \sum_{i=0}^m \binom{d}{i}$. This leads to the following approximate surrogate model for $f$:
\begin{equation}
	\widehat{f}_{\alpha}(x) = \sum_{i \in [p]} \alpha_{i} \psi_{i}(x). 
\end{equation}
We employ the latter representation as the surrogate model in our proposed algorithm.

%$d$: the number of boolean variables,
%$n$: sum constraint,
%$m$: max degree of monomials,
%$p$: the number of weights/coefficients/monomials,
%$t$: time step

\subsection{The \texttt{COMEX} Algorithm}
	
% Overview of the algorithm
Motivated by the properties of the hedge algorithm \cite{arora2012multiplicative}, we adopt an exponential weight update rule for our surrogate model. More precisely,  we maintain a pool of monomials where each monomial term plays the role of an expert. In particular, we are interested in finding the optimal Fourier coefficient $\alpha_i$ for the \textit{monomial expert} $\psi_i$. Note that exponential weights are non-negative, while the Fourier coefficients could be either negative or positive. Following the same approach as sparse online linear regression literature \cite{KIVINEN1997}, we maintain two non-negative coefficients for each Fourier coefficient $\alpha_i^t$ at time step $t$: $\alpha_{i, +}^t$ and $\alpha_{i, -}^t$. The value of the Fourier coefficient is then obtained via the subtraction $\alpha_i^t = (\alpha_{i, +}^t - \alpha_{i, -}^t)$.

% Description of the algorithm
More specifically, our algorithm works in the following way. We initialize the monomial coefficients $\alpha_{i, -}$ and $\alpha_{i, +}$ ($\forall i \in [p]$) with a uniform prior. In each time step $t$, the algorithm produces a sample point $x_t$ via Simulated Annealing (SA) over its current estimate for the Fourier representation $\widehat{f}_{\alpha^t}$ with Fourier coefficients $\alpha^t$. We then observe the black-box function evaluation $f(x_t)$ for our query $x_t$. This leads to a mixture loss $\ell_t$ which is equal to the difference between the evaluations obtained by our estimate model and the black-box function. This mixture loss, in turn, leads to the individual losses $\ell_i^t = 2 \, \lambda \, \ell_t \, \psi_i(x_t)$ for the monomial experts $\psi_i: \forall i \in [p]$. Finally, we update the current estimate for the Fourier coefficients $\alpha^t$ via the exponential weight update rule, incorporating the incurred losses. We repeat this process until the stopping criteria are met. Note that we use the anytime learning rate schedule of \cite{adaptive_EG}, which is a decreasing function of time $t$ (see Appendix \ref{app:learning_rate} for more details). A summary of the proposed algorithm, which we refer to as \textit{Combinatorial Optimization with Monomial Experts} (\texttt{COMEX}), is given in Algorithm \ref{algo:COMEX}. \\

% % Learning rate description
% We use the anytime learning rate schedule of \cite{adaptive_EG}, which is a decreasing function of time $t$. {\color{red}KAR:The learning rate choice below is sudden. Should we relegate it to the supplement? }. The learning rate at time $t$ is given by:
% \begin{equation}
% 	\eta_t = \min \bigg \{ \frac{1}{e_{t-1}}, c \sqrt{\frac{\ln{(2 \, p)}}{v_{t-1}}} \bigg \},
% \end{equation}
% where $c \overset{\Delta}{=} \sqrt{2(\sqrt{2} - 1)/(\exp(1)-2)}$ and 
% \begin{align*}
%     z_{j, t}^{\gamma} &\overset{\Delta}{=} - 2 \, \gamma \, \lambda \, \ell_t \, \psi_j(x_t) \\
%     e_t &\overset{\Delta}{=} \inf_{k \in \mathbb{Z}} \bigg \{ 2^k: 2^k \geq \max_{s \in [t]} \max_{ \substack{j, k \in [p] \\ \gamma, \mu \in \{-, +\}}} | z_{j, s}^{\gamma} - z_{k, s}^{\mu} |  \bigg \} \\
%     v_t &\overset{\Delta}{=}  \sum_{s \in [t]} \sum_{\substack{j \in [p] \\ \gamma \in \{-, +\}} } \alpha_{j, s}^{\gamma} \bigg ( z_{j, s}^{\gamma} - \sum_{\substack{k \in [p] \\ \mu \in \{-, +\}}} \alpha_{k, s}^{\mu} z_{k, s}^{\mu} \bigg )^2.
% \end{align*}

% Learning Graphical models (Ising Models/MRFs) \cite{MRF}

\begin{algorithm}[!t]
\caption{Combinatorial Optimization with Monomial Experts}
\begin{algorithmic}[1]
\STATE \textbf{Inputs:} sparsity $\lambda$, maximum monomial order $m$
\STATE $t = 0$
\STATE $\forall \gamma \in \{-, +\} \: \textrm{and} \: \, \forall i \in [p]: \alpha^t_{i, \gamma} = \tfrac{1}{2 p}$
% \FOR {$t = 0:T-1 $}
\REPEAT
\STATE $x_t \sim \widehat{f}_{\alpha^t} \;$ via Algorithm \ref{algo:SA}
\STATE Observe $\, f(x_t)$
\STATE $\widehat{f}_{\alpha^t}(x) \gets \sum_{i \in [p]} \big (\alpha^{t}_{i, +} - \alpha^{t}_{i, -} \big )~\psi_i(x) $.
\STATE $\ell^{t+1} \gets \widehat{f}_{\alpha^t}(x_t) - f(x_t)$
\FOR {$i \in [p] \: \textrm{and} \: \gamma \in \{-, +\}$}
\STATE $\ell_i^{t+1} \gets 2 \, \lambda \, \ell^{t+1} \, \psi_i(x_t)$
\STATE $\alpha^{t+1}_{i, \gamma} \gets \alpha^{t}_{i, \gamma} \exp \big (- \,\gamma \, \eta_t \, \ell_i^{t+1} \big)$
\STATE $\alpha^{t+1}_{i, \gamma} \gets \lambda \cdot \tfrac{\alpha^{t+1}_{i, \gamma}}{\sum_{\mu \in \{-, +\}} \sum_{j \in [p]} \alpha^{t+1}_{j, \mu}}$
\ENDFOR
\STATE $t \gets t + 1$
\UNTIL{Stopping Criteria}
\RETURN $\widehat{x}* = \arg\min_{\{x_i : \, \forall i \in [t]\}} f(x_i)$
\end{algorithmic}
\label{algo:COMEX}
\end{algorithm}

% \begin{algorithm}[!t]
% \SetAlgoLined
% \DontPrintSemicolon
% \KwIn{sparsity $\lambda$, maximum monomial order $m$}
% $t = 0$ \\
% $\forall \gamma \in \{-, +\} \: \textrm{and} \: \, \forall i \in [p]: \alpha^t_{i, \gamma} = \tfrac{1}{2 p}$ \\
% \Repeat{Stopping Criteria}{
% $x_t \sim \widehat{f}_{\alpha^t} \;$ via Algorithm \ref{algo:SA} \\
% Observe $\, f(x_t)$ \\
% $\widehat{f}_{\alpha^t}(x) \gets \sum_{i \in [p]} \big (\alpha^{t}_{i, +} - \alpha^{t}_{i, -} \big )~\psi_i(x) $ \\
% $\ell^{t+1} \gets \widehat{f}_{\alpha^t}(x_t) - f(x_t)$ \\
% \For{$i \in [p] \: \textrm{and} \: \gamma \in \{-, +\}$}{
% $\ell_i^{t+1} \gets 2 \, \lambda \, \ell^{t+1} \, \psi_i(x_t)$ \\
% $\alpha^{t+1}_{i, \gamma} \gets \alpha^{t}_{i, \gamma} \exp \big (- \,\gamma \, \eta_t \, \ell_i^{t+1} \big)$ \\
% $\alpha^{t+1}_{i, \gamma} \gets \lambda \cdot \tfrac{\alpha^{t+1}_{i, \gamma}}{\sum_{\mu \in \{-, +\}} \sum_{j \in [p]} \alpha^{t+1}_{j, \mu}}$ \\
% }
% $t \gets t + 1$ \\
% }
% % \Return{$\widehat{x}* = \arg\min_{\{x_i : \, \forall i \in [t]\}} f(x_i)$}
% \Return $\widehat{x}* = \arg\min_{\{x_i : \, \forall i \in [t]\}} f(x_i)$
% \caption{{\fontsize{7.785}{20}\selectfont Combinatorial Optimization with Monomial Experts}}
% % \caption{\mbox{Combinatorial Optimization with Monomial Experts}}
% \label{algo:COMEX}
% \end{algorithm}

\textbf{Theoretical Insights}: Let $D_{\mathrm{KL}}(p||q)$ denote the KL divergence between two distributions $p$ and $q$, i.e. $D_{\mathrm{KL}}(p||q) = \sum_i p_i \log\big(\tfrac{p_i}{q_i}\big)$. We can show that the KL-divergence between the estimate and true Fourier coefficients decreases over time, assuming that the true Fourier coefficients $\alpha^*$ are non-negative, and form a distribution, i.e. $\sum_i \alpha^*_i = 1$. Define $\phi_t = D_{\mathrm{KL}} (\alpha^*||\alpha^t)$ as the KL divergence between $\alpha^t$ and $\alpha^*$, where $\alpha^t$ are the estimates of Fourier coefficients at time $t$. With respect to Algorithm \ref{algo:COMEX}, $\alpha_i^t=\alpha_{i,+}^t$ and $\alpha_{i,-}^t=0$ in this case.
\begin{lemma}
\label{lemma_positive}
The exponential weight update at any time step $t$ for the Fourier coefficients $\alpha^t$, under the above stated assumption of non-negativity of the true Fourier coefficients $\alpha^*$, yields 
$$\phi_{t-1} \geq \phi_t + \eta \, 2 \, \lambda \, \big (\hat{f}_{\alpha_t}(x_t)-f(x_t) \big)^2 -  \eta^2$$
for $\eta < \frac{1}{8 \lambda}$.
\end{lemma}
\begin{proof}
Using Lemma 4.1 of \cite{Moritz2010}, for each exponential weight update at step $t$ where 
$\alpha_i^t=\alpha_{i,+}^t$ and $\alpha_{i,-}^t=0$, we have $\phi_{t-1} - \phi_{t} \geq \eta \, \langle r^t, \alpha^{t-1} - \alpha^* \rangle - \eta^2$ (for $0<\eta < \frac{1}{8 \lambda}$), where $r_t$ is the vector of individual losses, i.e. $r^t_i = \ell^t_i$. As a result, we only need to show that $\langle r^t, \alpha^{t-1} - \alpha^* \rangle$ is always greater than or equal to zero, since the value of $\eta$ can be chosen to be suitably small:
\begin{align*}
    \langle r^t, \alpha^{t-1} - \alpha^* \rangle &= \sum_i \ell_i^t \, ( \alpha^{t-1}_i - \alpha^*_i ) \\ 
    &= \sum_i 2  \lambda \ell^t \, \psi_i(x_t) \, ( \alpha^{t-1}_i - \alpha^*_i ) \\ 
    &= 2 \lambda  \, \ell^t \, \sum_i ( \alpha^{t-1}_i - \alpha^*_i ) \psi_i(x_t) \\
    &= 2 \lambda  \, \bigg ( \sum_i ( \alpha^{t-1}_i - \alpha^*_i ) \psi_i(x_t) \bigg ) ^ 2  \\
    & = 2 \lambda (\hat{f}_{\alpha_t}(x_t)-f(x_t))^2 \geq 0.
\end{align*}
This proves the Lemma.
\end{proof}

For the generalization of this result to the case of Fourier coefficients with arbitrary signs, see Remark \ref{remark_general} in Appendix \ref{app:lemma_extension}.

% \textbf{Remark:} 
\begin{remark} The above Lemma shows that for a small enough $\eta$, $\phi_{t-1} -\phi_t \geq 0$ for any evaluation point $x_t$.
This shows that for a sufficiently small learning rate $\eta$, irrespective of the evaluation point $x_t$, there is a potential drop in the distance between the true and estimated coefficients after the exponential weight update at time $t$. This observation motivates our surrogate model and the deployment of the exponential weight update rule.
\end{remark}

% This shows that for a small enough learning rate $\eta$, irrespective of the evaluation point $x_t$, there is a potential function measuring the distance between the true weights and the estimated one that decreases after the update at time $t$. This motivates our surrogate model and the exponential weight update.

\subsection{Acquisition Function}
Our acquisition function is designed to minimize $\hat{f}$, the current estimate, in a way that allows some exploration. To this end,
we employ a version of simulated annealing (SA) \cite{Kirkpatrick1983, SA} as our acquisition function that uses the offline evaluations of the surrogate model. SA consists of a discrete-time inhomogeneous Markov chain, and is used to address discrete optimization problems. The key feature of simulated annealing is that it provides a means to escape local optima by allowing probabilistic hill-climbing moves in hopes of finding a global optimum. Although SA is not sample efficient in practice, and as we will see in Section \ref{sec:experiments} is not suitable for optimization of black-box functions, it can be set up in conjunction with a surrogate model. 

Define the neighborhood model $\mathcal{N}$ for the unconstrained problem as:
\begin{equation}
    \mathcal{N}(x_t) \gets \{x_i: d_H(x_i, x_t) = 1 \; \textrm{and} \; x_i \in \{0, 1\}^d \},
\end{equation}
where $d_H(x_i, x_t)$ designates the Hamming distance between $x_i$ and $x_t$.
Also, we define the following neighborhood model for the sum-constrained problem:
\begin{equation}
    \label{constrained_neighborhood}
    \mathcal{N}(x_t) \gets \{x_i: d_H(x_i, x_t) = 2 \; \textrm{and} \; x_i \in \mathcal{C}_n \}.
\end{equation}

Algorithm \ref{algo:SA} presents the simulated annealing for the latter two combinatorial structures, where $s(t)$ is an annealing schedule, which is a non-increasing function of $t$. We use the annealing schedule suggested in \cite{SA}, which follows an exponential decay given by $s(t) = \exp(-\omega t/d)$, where $\omega$ is a decay parameter. \\

%\subsection{Analysis for Exponential Acquisition Function}
\textbf{Analysis for Exponential Acquisition Function}:
We used a simulated annealing based acquisition function on the surrogate model $\hat{f}(x)$. This model is very difficult to analyze. Instead, we analyze an exponential acquisition function given by 
$$x \sim \frac{\exp\left(-\nicefrac{\hat{f}_{\alpha_t}(x)}{T}\right)}{\sum_{x\in \{-1,1\}^d} \exp \left(-\nicefrac{\hat{f}_{\alpha_t}(x)}{T}\right)}$$
where $T$ is the temperature. Let the p.m.f of this acquired sample be $\hat{P}_{\alpha_t}(x)$.
If we had access to the actual function $f$, we would use the acquisition function: 
$$x \sim \frac{\exp\left(-\nicefrac{f(x)}{T}\right)}{\sum_{x\in \{-1,1\}^d} \exp \left(-\nicefrac{f(x)}{T}\right)}.$$
Let the p.m.f of this acquired sample be $P(x)$. We emphasize that, given explicit access to $f$ (white-box), one would simply repeatedly acquire samples using the acquisition function for $f$. In our black-box algorithm, we use the surrogate model to acquire samples. 

Now, we show a result that implies the following under some additional technical condition: \textit{Until the acquisition function based on $\hat{f}_{\alpha_t}$ yields samples which are close in KL divergence to samples yielded by the acquisition function based on $f$, average $\phi_{t-1} - \phi_{t}$ (as in Lemma \ref{lemma_positive}) is large.}

In other words, if the acquisition function for our algorithm is far from the white-box acquisition function, then non-trivial learning of $f$ happens, i.e. $\alpha_t$ moves closer to $\alpha^{*}$ at least by a certain amount.

Let $\hat{Z}= \sum_{x} \exp \left( - \hat{f}_{\alpha_t}(x)/T \right)$ be the partition function. Similarly, let $Z$ be the partition function associated with $P(x)$.

\begin{theorem}
\label{thm:Acquisition}
 Let $-1 \leq \hat{f}_{\alpha_t}(x), f(x) \leq 1$. If at any time $t$, and for a temperature $T$, we have for some $\epsilon>0$: \[ \left \lvert D_{\mathrm{KL}}(\hat{P}_{\alpha_t} \lVert P) -  \log \left( \frac{ Z}{\hat{Z}}\right)  \right \rvert \geq \epsilon,\]
 then $\mathbb{E}_{\hat{P}_{\alpha_t}} [\phi_{t-1}-\phi_t] \geq  2 \, \eta \, \lambda \, \epsilon^2 T^2 - \eta^2$. Here, $D_{\mathrm{KL}}$ is defined with respect to $\log_e$ for convenience.
\end{theorem}
\begin{proof}
The proof is in the supplement.
\end{proof}

\begin{remark} Note that the condition in the theorem definitely implies the condition that $D_{\mathrm{KL}}(\hat{P}_{\alpha_t} \lVert P) >0$.
\end{remark}

% \textbf{Remark:} We note that in our schedule for simulated annealing we set $T=\exp(-t/d)$ in Algorithm \ref{algo:SA} at time $t$. 

\begin{algorithm}[!t]
\caption{Simulated Annealing for Combinatorial Constraints}
\begin{algorithmic}[1]
\STATE \textbf{Inputs:} surrogate model $\widehat{f}_{\alpha_t}$, neighborhood model $\mathcal{N}$, Constraint Set $\mathcal{C}$
\STATE $t = 0$
\STATE Initialize $\; x_0 \in \mathcal{C}$
% \FOR {$t = 0:T-1 $}
\REPEAT
\STATE $z \sim \texttt{unif}\big(\mathcal{N}(x_{t}) \big)$
\IF {$\widehat{f}_{\alpha_t}(z) \leq \widehat{f}_{\alpha_t}(x_{t})$}
\STATE $x_{t+1} \gets z$
\ELSIF{$\texttt{unif}(0, 1) \leq \exp\bigg(-\tfrac{\widehat{f}_{\alpha_t}(z) - \widehat{f}_{\alpha_t}(x_{t})}{s(t)}\bigg)$}
\STATE $x_{t+1} \gets z$
\ELSE
\STATE $x_{t+1} \gets x_t$
\ENDIF
\STATE $t \gets t + 1$
\UNTIL{Stopping Criteria}
\RETURN $x_t$
\end{algorithmic}
\label{algo:SA}
\end{algorithm}

\subsection{Computational Complexity}

The computational complexity per time step for model learning in the proposed algorithm is in $\mathcal{O}(p) = \mathcal{O}(d^m)$, which is linear in the number of Fourier coefficients. More importantly, the complexity of the proposed learning algorithm is independent of the number of function evaluations (i.e. time step $t$). We also note that the complexity of the simulated annealing is in $\mathcal{O}(d^2)$; therefore, the overall complexity of the algorithm remains $\mathcal{O}(p)$ for $m \geq 2$.

%\subsection{Extension to Categorical Variables}

%In this work we have exclusively focused  on the problem of black-box optimization over the boolean hypercube. It is noteworthy that the proposed algorithm can be extended to address optimization problems over ordinal and categorical variables in a straightforward fashion. Following  the same approach suggested in \cite{BOCS}, we replace each categorical or ordinal variable $e_i$ of cardinality $\ell_i$ with $\ell_i$ new variables $\forall j \in [\ell_i]: e_{ij} \in \{-1, 1\}$; we then enforce the following constraint: For any $i$ exactly one of the variables $e_{ij}$ takes the value of $1$; the remaining variables are set to $-1$. This constraint is implemented during the function acquisition process, where SA would only search through the neighbors satisfying such a constraint. The model learning process is performed over the entire set of extended variables $\{e_{ij}: \forall i \in [d] \, \texttt{and} \, j \in [\ell_i]\}$ in the same fashion as the boolean case.

\section{Experiments and Results}
\label{sec:experiments}

In this section, we evaluate the performance of the proposed algorithm in terms of simple regret as well as average computational time required to draw a sample for each function evaluation. We consider four problems: two unconstrained combinatorial problems (Sparsification of Ising Models and Contamination Control) as well as two combinatorial optimization problems with a sum constraint (Noisy $n$-Queens and optimal defect pattern in 2D nanomaterials). The latter problem is adopted from a real-world scenario; it takes advantage of a genuine energy evaluator (as black-box function) recognized in  the molecular modeling literature, and has crucial real-life applications in designing nanomaterials utilized in nanoscale electronic devices \cite{defect_dynamics}. The two unconstrained combinatorial problems have also been considered in \cite{BOCS} and \cite{COMBO}.

We investigate the performance of different algorithms in two settings: (i) finite evaluation-budget regime, and (ii) finite time-budget regime. In the former case, we assume that each algorithm is given a fixed evaluation budget and has access to an unlimited time budget. In this setting, we consider problems with a relatively small number of variables. In the latter case, we assume that each algorithm, in addition to an evaluation budget, has a limited time budget and reports the minimum obtained within that time frame. This scenario is particularly relevant in problems with moderate numbers of variables, since the computational cost for state-of-the art algorithms are prohibitive, which makes it impossible in practice to draw a large number of samples for function evaluation.

The results are compared against two baselines, random search (RS) \cite{RS} and simulated annealing (SA) \cite{Kirkpatrick1983, SA}, as well as two state-of-the-art algorithms, BOCS \cite{BOCS} and COMBO \cite{COMBO}. We use the BOCS-SA version of BOCS, as opposed to BOCS-SDP, since the former version is computationally less expensive; as such, its use would make more sense than BOCS-SDP in the finite time-budget setting. In addition, the BOCS-SA algorithm can be adapted to optimization problems with a sum constraint in a straightforward fashion.

%since only the former version is able to address interactions of higher orders; hence, it achieves improved results in the infinite time-budget setting when such interactions are present in the data, as shown in \cite{BOCS}. In addition, the BOCS-SA algorithm can be adapted to optimization problems with a sum constraint in a straightforward fashion. Finally, BOCS-SA is computationally less expensive; as such, its use would also make more sense than BOCS-SDP in the finite time-budget setting.

All the results are averaged over $10$ runs. We measure the performance of different algorithms in terms of the mean over simple regrets $\pm$ one standard error of the mean. We run the experiments on machines from the Xeon E5-2600 v3 family. The function evaluations in all the experiments are linearly mapped from the original interval $[\textrm{min}(f), \textrm{max}(f)]$ to the target interval $[-1, 1]$. Hence, the function value of $-1$ corresponds with the desired minimum. In many cases, we know a lower bound on $\min(f)$ and an upper bound on $\max(f)$ that enables us to scale to $[-1,1]$. In other cases where the lower bound on $\min(f)$ is unknown analytically, we fix a universal level which is smaller than all observed evaluations and compare all algorithms over all runs to this fixed level.

The sparsity parameter $\lambda$ of our proposed algorithm is set to $1$ in all the experiments. In our experiments, the algorithm was relatively insensitive to the choice of this parameter. Note that BOCS and COMBO also include sparsity/regularization and exploration parameters for their respective algorithms, the choice of which did not seem to impact the outcome noticeably, in a similar fashion to our algorithm. 

Finally, we note that COMBO is order agnostic, while our proposed algorithm as well as BOCS take the maximum monomial degree $m$ as an input parameter. The maximum order $m = 2$ or $m = 3$ is deployed for our algorithm in the experiments. In particular, as shown in \cite{BOCS} and verified in our experiments, the sparsification of the Ising models as well as the contamination control problem have natural interactions of orders higher than two among the variables. As such, we set $m = 3$ in the latter two problems. In the remaining experiments, a maximum order of $m = 2$ is utilized. We set the maximum order of BOCS to $2$ in all the experiments, due to its excessively high computational cost at $m=3$.

\subsection{Sparsification of Ising Models}

Let $p$ be a zero-field Ising model with the probability distribution of $p(z) = \tfrac{1}{Z_p} \exp(z^T J^p z)$, where $z \in \{-1, 1\}^n$, $Z_p = \sum_z \exp(z^T J^p z)$ is the partition function, and $J^p \in \mathbb{R}^{n \times n}$ is a symmetric interaction matrix. The aim of the Ising sparsification problem is to find a sparse approximation model $q_x(z) = \tfrac{1}{Z_q} \exp(z^T J^{q_x} z)$ such that $\forall i, j \in [n]: J^{q_x}_{i, j} = x_{i, j} J^{p}_{i, j}$, where $x_{i, j} \in \{0, 1\}$ are decision variables. The solution to this problem is given by
\begin{equation}
	x = \arg\min_{x \in \{0, 1\}^d} D_{\mathrm{KL}}(p||q_x) + \lambda \|x\|_1
\end{equation}
where $D_{\mathrm{KL}}(p||q_x)$ is the KL divergence of $p$ and $q_x$, and $\lambda$ is a regularization parameter. The KL divergence is expensive to compute; in addition, an exhaustive search requires $2^d$ function evaluations, which is infeasible in general.

We follow the same experimental setup as in \cite{BOCS} and \cite{COMBO}, where we have an Ising model with $n = 16$ nodes and $d = 24$ interactions. The values of the interactions are sampled uniformly at random from the interval $[0.05, 5]$. The simple regret of various algorithms for the regularization parameter $\lambda = 0.01$ is depicted in Figure \ref{fig:ising}. The proposed algorithm is able to hit the minimum found by both COMBO and BOCS, although it requires more function evaluations to achieve that feat. However, we point out that, as depicted in Table \ref{times}, the proposed algorithm only takes $0.1$ seconds per time step on average as opposed to $47.7$ and $78.4$ seconds per time step taken for BOCS and COMBO, respectively. In particular, both BOCS and COMBO are computationally far more expensive than the black-box function, whose average evaluation cost for each query is $2.24$ seconds. Despite its poor initial performance, SA is also able to virtually reach the same minimum value as the latter three algorithms.

We note that the complexity of the Ising sparsification problem grows exponentially with $d$; hence, it is computationally infeasible to obtain black-box evaluations for moderately large numbers of variables; hence, we only considered the finite evaluation-budget setting for this particular problem.

\begin{figure}[t]
\center
\includegraphics[width=\linewidth]{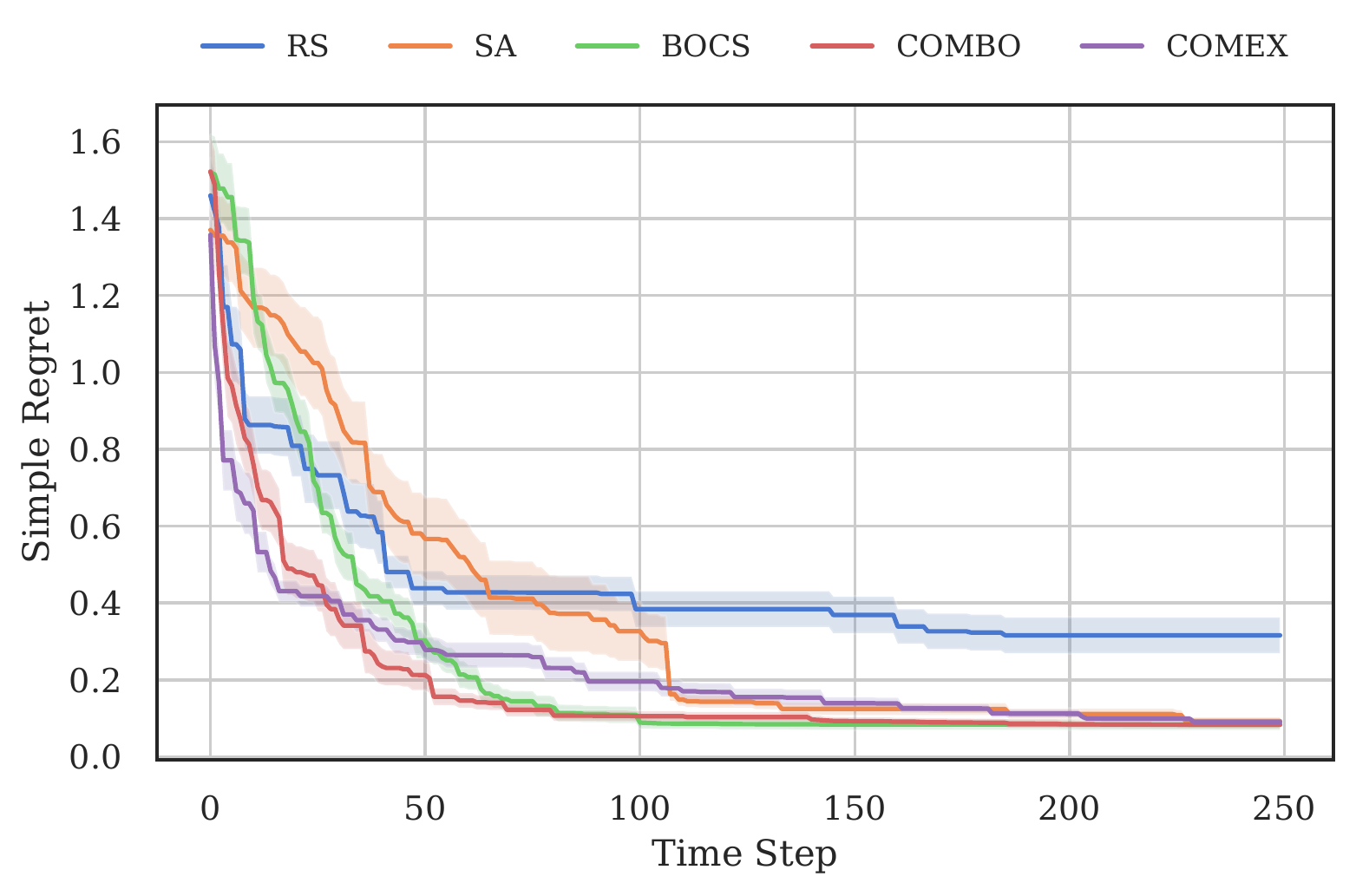} 
\caption{\small \label{fig:ising}
Simple regret for the Ising Sparsification problem.}
\end{figure}

\subsection{Contamination Control}

The contamination control in food supply chain \cite{contamination_control} is an optimization problem with $d$ binary variables representing stages that can be contaminated with pathogenic microorganisms. At each time step, one can intervene at each stage of the supply chain to reduce the contamination risk by a random rate (which follows a beta distribution) and incur an associated cost. The goal is to minimize the overall food contamination while minimizing the total prevention cost. As such, the minimization is carried out with respect to the set of decision variables $x \in \{0, 1\}^d$ incorporating an $\ell_1$ regularization term with a regularization parameter $\lambda$. 

Following the experimental setting of \cite{COMBO} and \cite{BOCS}, we initially set $d = 21$ and $\lambda = 0.01$. The results in terms of simple regret are shown in Figure \ref{fig:contamination}. As we can see from this Figure, COMBO outperforms the rest of the algorithms in that it is able to find the minimum in just over $100$ function evaluations on average. Despite its initially large regret results, BOCS is also able to find the minimum in just over $150$ function evaluations. The proposed \texttt{COMEX}  algorithm is also competitive and is able to find the minimum in just over $200$ function evaluations. Note that SA and especially RS were not able to achieve the minimum in $250$ function evaluations. Finally, we point out that the proposed algorithm takes a fraction of time required by BOCS and COMBO in order to draw evaluation samples, as shown in Table \ref{times}.

We then increase the dimensionality of the problem to $d = 100$ variables. Due to the high dimensionality of this problem, drawing samples from both COMBO and BOCS becomes computationally expensive. Therefore, in addition to the evaluation budget of $1000$, we set a finite time budget of $24$ hours and run the experiments until at least one of the budget constraints is attained. Simple regret results are depicted in Figure \ref{fig:contamination}. In this setting, BOCS is only able to draw $\approx 150$ samples, while COMBO exceeds the time budget at around $100$ samples. On the other hand, the proposed algorithm is able to produce $1000$ samples quickly and approach the minimum function value. Considering the high dimensionality of this data, RS produces poor results, whereas SA incurs an eventual simple regret of $0.2$ on average. Finally, we note that  \texttt{COMEX} is over $100$ times faster than both BOCS and COMBO, as depicted\footnote{Note that the average computation times reported in Table \ref{times} correspond only with producing a point via the algorithm, whereas the $24$-hour budget includes both the black-box evaluation cost and the computation time of the algorithm.} in Table \ref{times}.

\begin{figure}[t]
\center
\includegraphics[width=\linewidth]{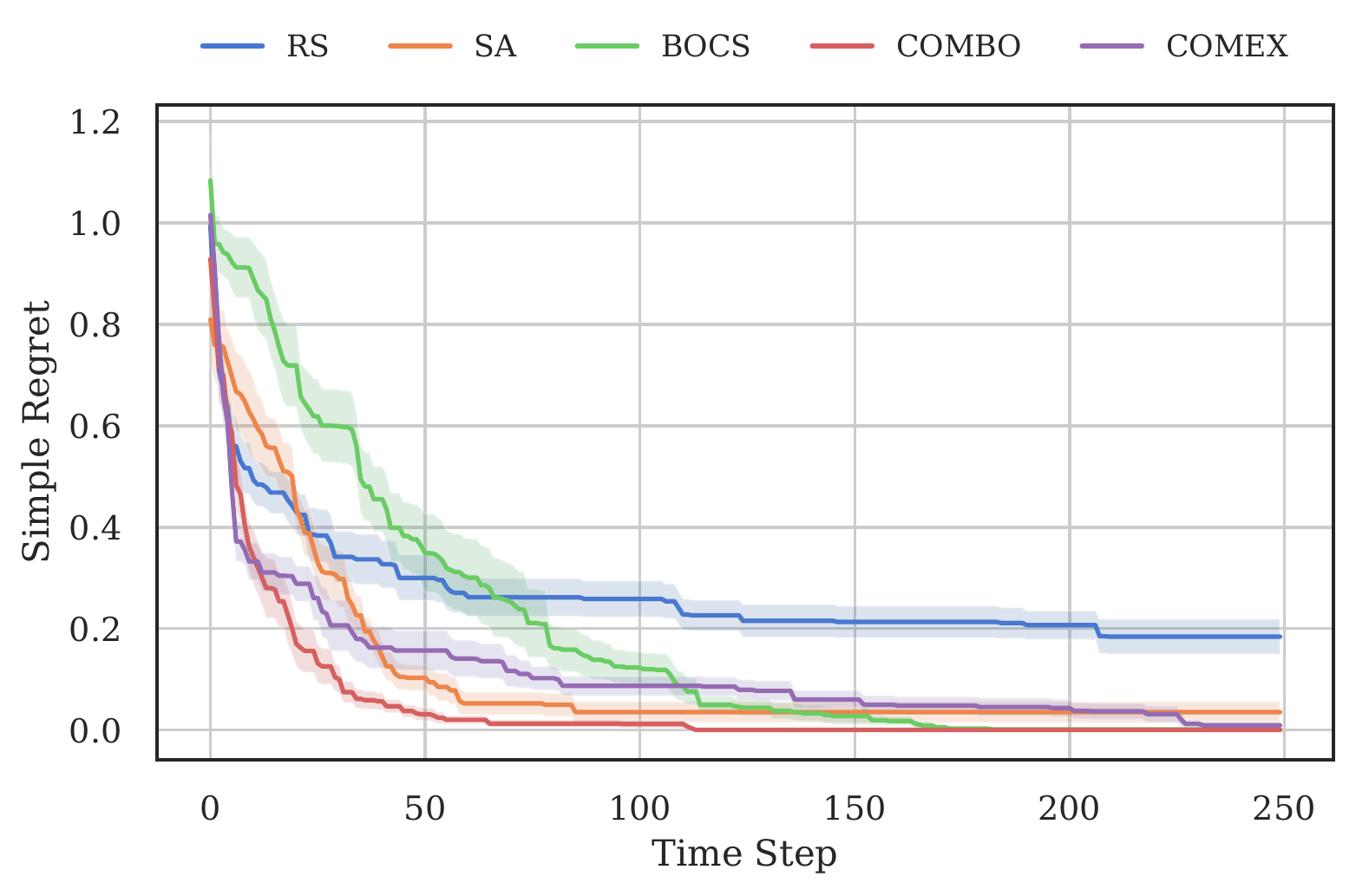} \\
\includegraphics[width=\linewidth]{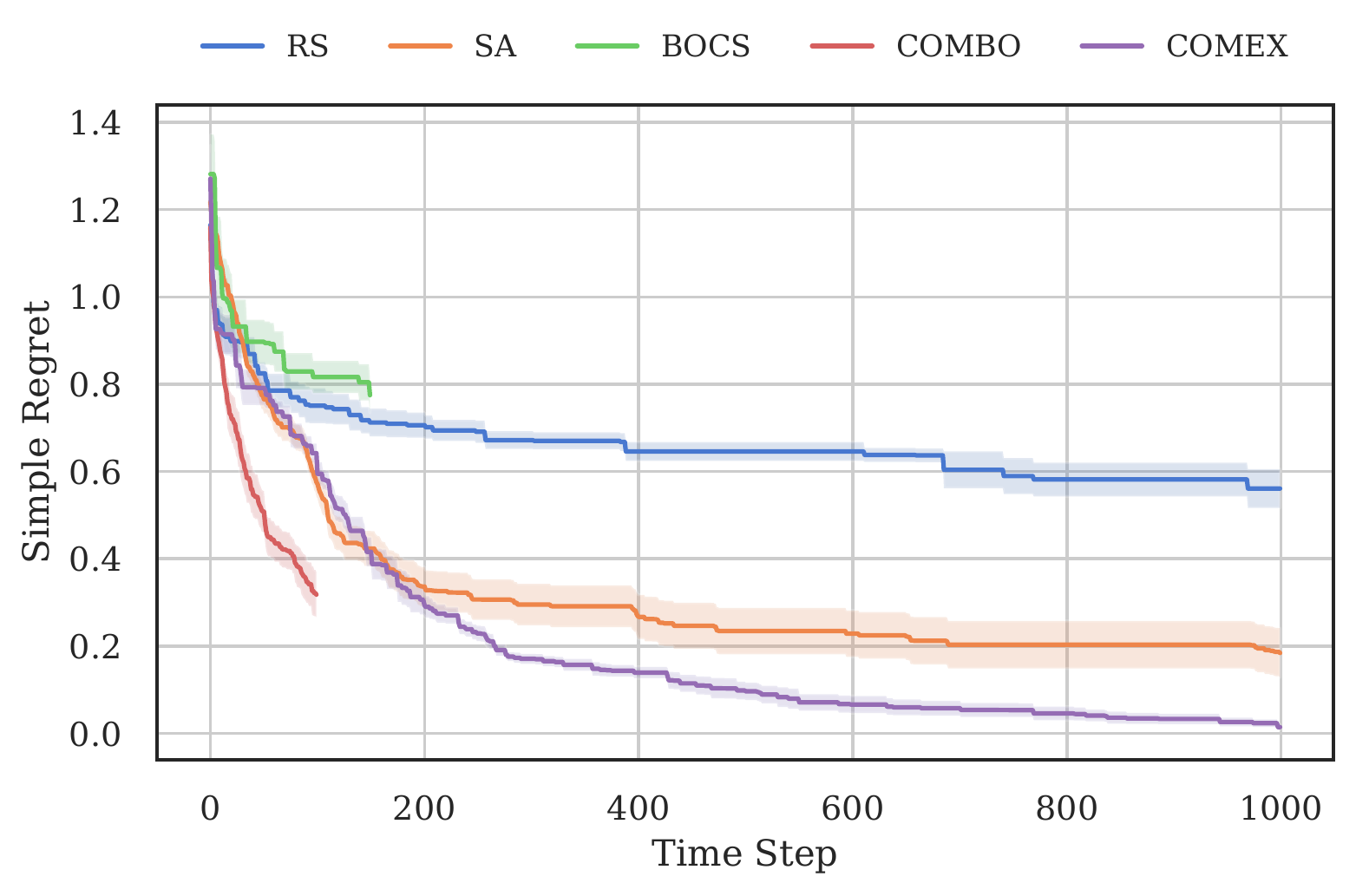} 
\caption{\small \label{fig:contamination}
Simple regret for the contamination control problem with: $d = 21$ (top), and $d = 100$ (bottom).}
\end{figure}

\begin{table}[t]
\caption{Average computation time per step (in Seconds) over different datasets for different algorithms}
\label{times}
\begin{center}
%\begin{small}
\small
% \footnotesize
\begin{sc}
\begin{tabular}{lccccr}
\toprule
\multicolumn{3}{c}{Black Box} & \multicolumn{3}{c}{Algorithm} \\
\cmidrule(lr){1-3}\cmidrule(lr){4-6}
Dataset & $d$ & Cost & BOCS & COMBO & \texttt{COMEX} \\
% \midrule
\cmidrule(lr){1-3}\cmidrule(lr){4-6}
N-Queens      & 49  & 0.001 & 202.1 & 336.7 & $\bm{0.09}$ \\
Contamination & 21  & 0.001 & 28.6 & 53.8 & $\bm{0.07}$ \\
Ising         & 24  & 2.24 & 47.7 & 78.4 & $\bm{0.10}$  \\
% \midrule
\cmidrule(lr){1-3}\cmidrule(lr){4-6}
N-Queens    & 144 & 0.001 & 401.28 & 722.05 & $\bm{2.87}$ \\
Contamination & 100 & 0.002 & 454.93 & 587.65 & $\bm{1.33}$ \\
Defect Dynamics  & 400 & 73.16 & 873.99 & 3869.92  & $\bm{65.5}$ \\
\bottomrule
\end{tabular}
\end{sc}
% \end{small}
\end{center}
\end{table}

\subsection{Noisy \texorpdfstring{$n$}{-}-Queens}

We next consider a constrained optimization problem where the search space is restricted to the combinatorial domain $\mathcal{C}_n$. We adapt the acquisition function of different algorithms to this constrained domain in a straightforward manner. More specifically, we modify the local neighborhood search in both SA (in BOCS as well as in our proposed algorithm) and graph local search (in COMBO) to the constrained domain $\mathcal{C}_n$ by restricting the neighborhood to data points with Hamming distance of two rather than one, as defined in \eqref{constrained_neighborhood}.

The $n$-queens problem is a commonly used benchmark in combinatorial optimization literature \cite{Mukherjee2015, Hu2003swarm, takenaka2000proposal, homaifar1992n}. This problem consists of finding the placement of $n$ queens on an $n  \times n$ chessboard so that no two queens share the same row, column, or diagonal \cite{Bell2009survey}. This problem can be formulated as a constrained binary optimization problem. We use binary variables $x_{ij}$ to represent the placement of a queen in each square position of the chessboard given by its row and column pair $(i,j)$, for $i, j \in [n]$. A placement of queens is then represented by a binary vector $x$ of length $d = n \times n$. Hence, a solution to the $n$-queens problem simultaneously meets the following constraints: 
\begin{itemize}
\item There is exactly one queen in each row $i \in [n]$:
	\begin{equation}\label{eq:queens_r} e_{\textrm{rows}}(x) = \sum_i ( \sum_j x_{ij} -1 )^2 = 0,  \end{equation}
\item There is exactly one queen in each column $j \in [n]$:
	\begin{equation}\label{eq:queens_c}  e_{\textrm{cols}}(x) = \sum_j ( \sum_i x_{ij} -1 )^2 = 0,  \end{equation}
\item There is at most one queen in each diagonal:
	\begin{equation}\label{eq:queens_d}  e_{\textrm{diags}}(x) =\sum_{\ell} \, \sum_{(i, j) \neq (k, h) \in \mathcal{D}_{\ell}} x_{ij}  x_{kh} = 0, \end{equation}
where $\mathcal{D}_{\ell}$ represents the set of all the elements in the $\ell$-th diagonal, and the first summation is taken over all the diagonals with at least one square position.
\end{itemize}
    
\noindent The non-negative quadratic terms in constraints \eqref{eq:queens_r}-\eqref{eq:queens_d} indicate deviations from the required number of queens in each row, column, and diagonal, respectively. Thus, if there exists a solution to the $n$-queens problem given by a binary vector $x$ satisfying all the constraints, the minimum of \begin{equation}\label{eq:queens_energy} f(x) = e_{\textrm{rows}}(x) + e_{\textrm{cols}}(x) + e_{\textrm{diags}}(x) \end{equation} must be achieved at zero, and vice versa. We know that for $n>3$ a solution to the $n$-queens problem indeed exists; therefore, minimizing $f(x)$ is equivalent to solving the constraints \eqref{eq:queens_r}--\eqref{eq:queens_d}. This allows the formulation of the problem as an unconstrained optimization one.

To provide a benchmark for the constrained optimization case, we add the redundant constraint that $\sum_i \sum_j x_{ij}= n$ to our formulation when generating samples, effectively reducing the search space to $\mathcal{C}_n$. Thus, for each problem of size $d$, we have $d = n \times n$ binary variables to optimize over, where the search space is constrained to binary vectors with $n$ ones.
We consider a noisy version of this problem, where the function evaluations from Equation \eqref{eq:queens_energy}, having been linearly mapped to the interval $[-1, 1]$, incur an additive Gaussian noise with zero mean and standard deviation of $\sigma = 0.02$.

First, we consider a smaller version of this problem with $n = 7$ and a finite evaluation budget of $250$ samples. In this experiment, all the algorithms are able to exhaust the evaluation budget within the allocated $24$-hour time frame. The results in terms of simple regret are depicted in Figure \ref{fig:nqueens}. As we can see from this figure, COMBO outperforms all the algorithms. BOCS performs only slightly better than RS. \texttt{COMEX} is a close second, while being able to run the experiment at a fraction of the time consumed by either COMBO or BOCS as indicated in Table \ref{times}. 

Next, we increase the size of the problem to $n = 12$ and enforce a finite time budget of $24$ hours. In this case, COMBO and BOCS are unable to use the evaluation budget within the allotted time frame, and manage to draw only $\approx 100$ and $\approx 150$ samples, respectively. The proposed algorithm, on the other hand, is able to take advantage of the full evaluation budget and outperforms the baselines by a significant margin, as shown in Figure \ref{fig:nqueens}. 

\begin{figure}[t]
\center
\includegraphics[width=\linewidth]{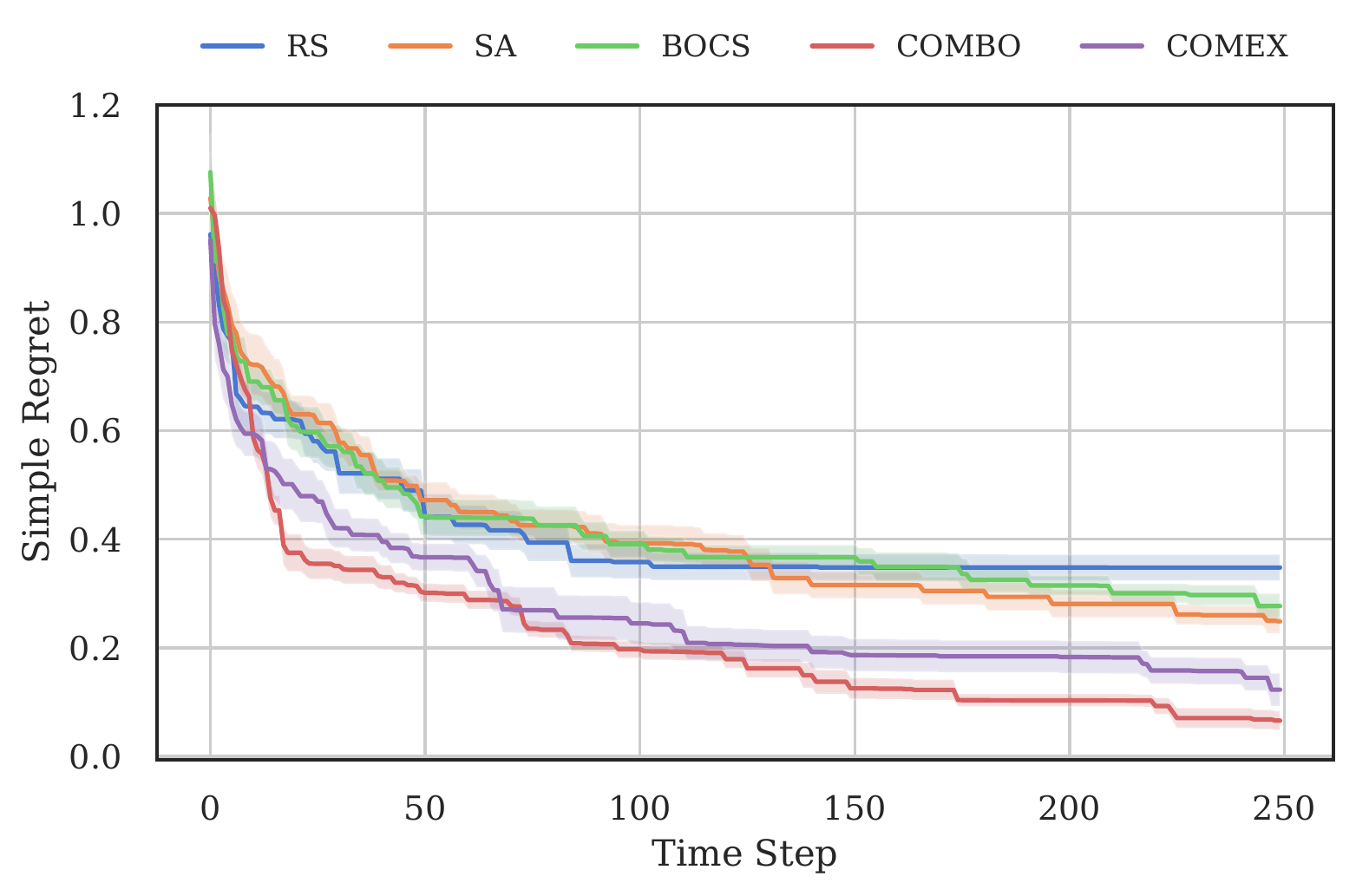} \\
\includegraphics[width=\linewidth]{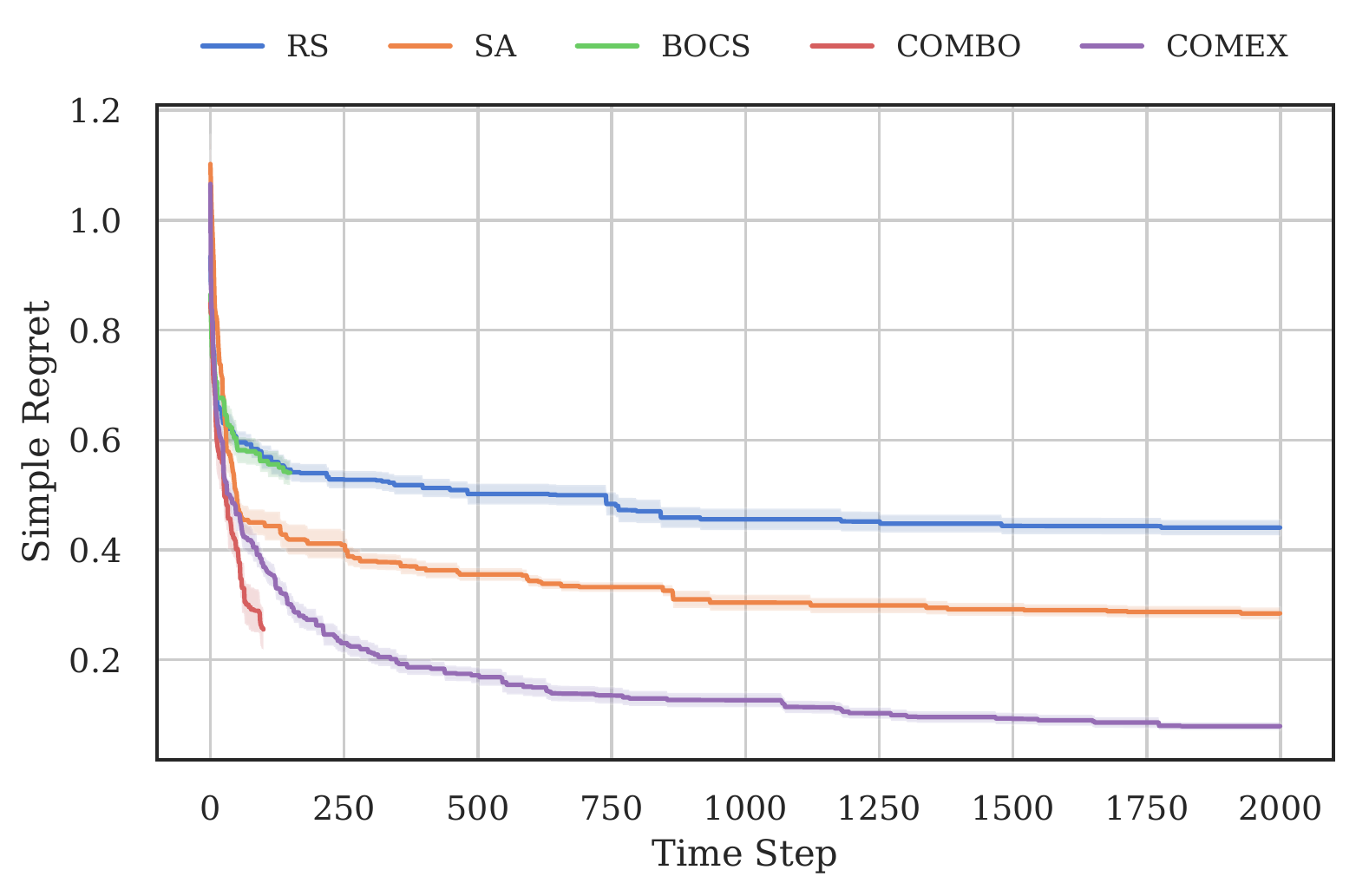} 
\caption{\small \label{fig:nqueens}
Simple regret for the noisy $n$-Queens problem with: $n = 7$ ($d = 49$) (top), and $n = 12$ ($d = 144$) (bottom).}
\end{figure}

%\subsection{Optimal Arrangement of Defect Structures in 2D Nanomaterials} % in semiconductors
\subsection{Optimal Arrangement of Defect Structures in Nanomaterials}
Physical, chemical, and optoelectronic properties of two-dimensional transition-metal dichalcogenides (TMDs), such as MoS$_2$, are governed by structural defects such as Sulfur vacancies. A fundamental understanding of the spatial distribution of defects in these low-dimensional systems is critical for advances
in nanotechnology. Therefore, understanding the dynamics of point defect arrangements at various vacancy concentrations is crucial, as those are known to play a key role in phase transformation governing nanoscale electronic devices \cite{defect_dynamics}. 
%Transition metal dichalcogenide (TMD) monolayers are atomically thin semiconductors which have attracted a lot of attention for applications in nanoscale electronic devices. In TMD's, various structural defects may either pre-exist or be introduced during preparation, processing, and transfer processes. Such defects, which immensely influence the physicochemical and electronic properties of the semiconductor, can be either as isolated points or in extended format. A fundamental understanding of the spatial distribution and dynamics of such defects in the TMD monolayers is essential for advances in nanotechnology. 

Given a two-dimensional grid of size $d$, the problem is to find the formation of a defect structure of size $n$ (corresponding to a given concentration factor) with the lowest energy, in which defects can be in either isolated or extended form (i.e. several defects next to each other) or a combination of both \cite{defect_dynamics}. Using  the reactive force field (ReaxFF) \cite{ostadhossein2017reaxff} within LAMMPS simulation package \cite{Plimpton1995lammps}, we are able to obtain an energy evaluation for each selection of the defect structure. However, such function evaluations are computationally expensive to acquire. Hence, we are interested in finding the low-energy defect structure with as few energy evaluations as possible. 

In our experiments, we deploy a 2-D MoS${_2}$ monolayer with a grid of size $d=400$, following the framework suggested in \cite{defect_dynamics}. In particular, we are interested in finding the optimal placement of $n = 16$ sulfur vacancies in the MoS$_2$ monolayer. Considering the moderately high dimensionality of this problem, the computational complexities of BOCS and COMBO render their applicability practically impossible, as it would take the latter algorithms several weeks to accumulate a few hundred candidate data points for energy evaluation purposes\footnote{With a $24$-hour time budget, BOCS managed to complete just over $50$ steps with a simple regret of $0.59$, whereas COMBO even failed to produce that many steps.}.

As can be observed in Figure \ref{fig:defect_dynamics}, the proposed \texttt{COMEX} algorithm outperforms the baselines, RS and SA, in identifying the optimal defect structure in the sample TMD grid of size $d=400$ over $500$ energy evaluations. In this experiment, since the exact value of the minimum energy is unknown, for the purpose of simple regret calculations, we pick a fixed universal energy level which is less than all the obtained function evaluations via all the algorithms in our experiments.

%\begin{figure*}[t]
%\center
%\includegraphics[width=.32\textwidth]{figures/simple_regret_lammps_16.pdf} \hfill
%\includegraphics[width=.32\textwidth]{figures/simple_regret_lammps_30.pdf}
%\caption{\small \label{fig:defect_dynamics}
%Simple regret for the discovery of defect dynamics problem with: $n = 16$ (top) and $n=30$ (bottom) ($d = 400$ in both cases).}
%\end{figure*}

\begin{figure}[t]
\center
\includegraphics[width=\linewidth]{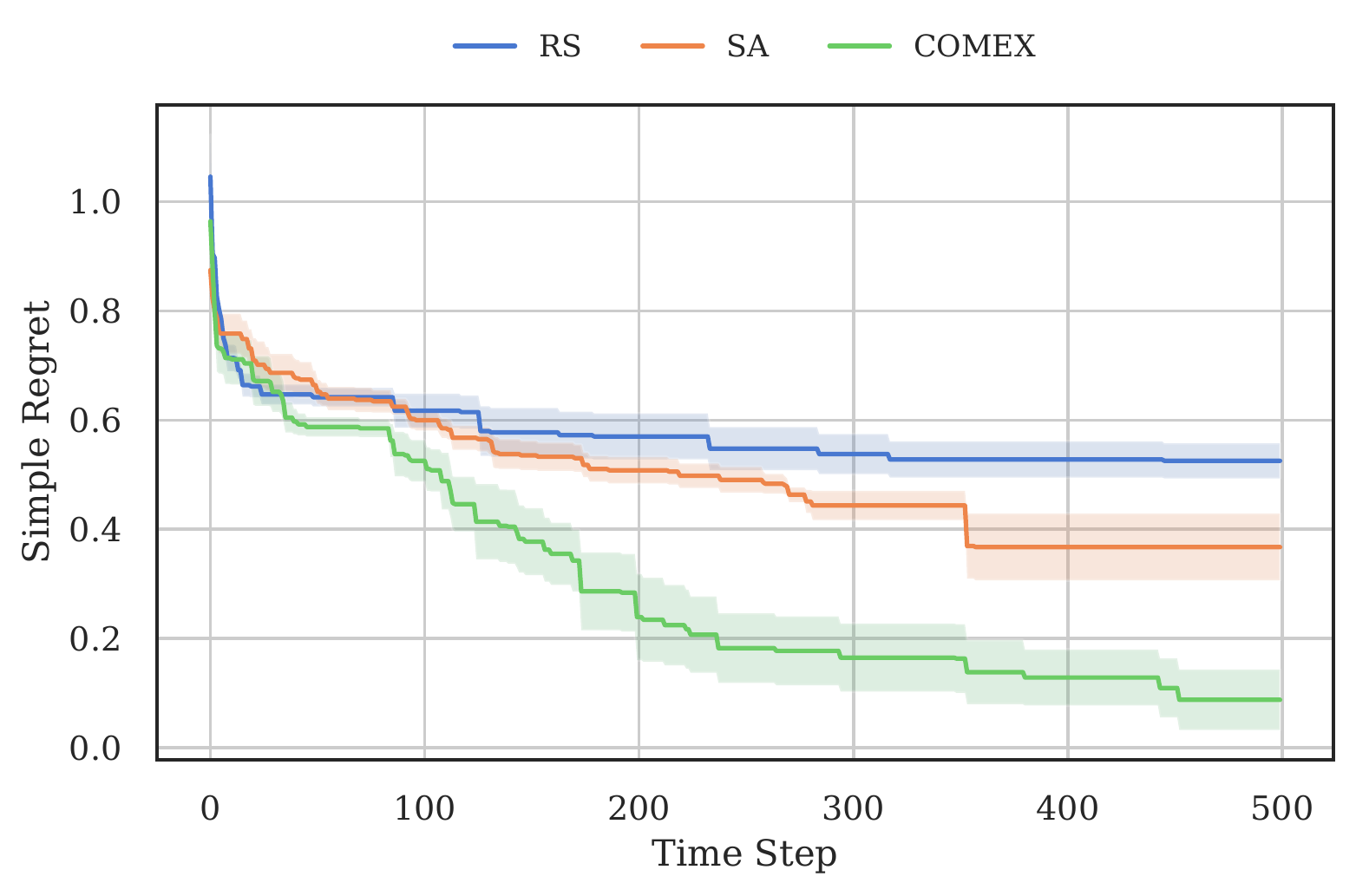}
\caption{\small \label{fig:defect_dynamics}
Simple regret for the optimal organization of point defect problem with $n = 16$ ($d = 400$).}
\end{figure}

\section{Future Work}

% Computational complexity for higher order interactions
As mentioned in Section \ref{sec:algorithm}, the computational cost (per time step) of the proposed \texttt{COMEX} algorithm is independent of the number of function evaluations and is linear in the number of monomial terms. This is a major improvement over the existing state-of-the-art algorithms in that the dependence on the number of function evaluations is completely eliminated from the complexity. Nevertheless, the complexity of the algorithm with respect to the number of variables grows polynomially, and could become expensive for problems with particularly higher orders of interactions incorporating a large number of variables. Therefore, an important direction for future work would involve investigating the possibility of improving this time complexity. %Specifically, we pose the following question: Is it possible to derive a linear time algorithm with respect to the number of variables, capable of capturing higher order interactions, while maintaining the independence on the number of function evaluations? 
We speculate that the proposed algorithm can be extended to accommodate such computational requirements via addition of fresh experts over time in an adaptive fashion. 
%In particular, the answer to this problem entails carefully designing priors for asleep and/or fresh experts.

% Improvements in acquisition function
In this work, we utilized a simple simulated annealing method as our acquisition function model and proved our results regarding acquisition through exponential sampling. Another avenue for future research is to develop more efficient strategies to model the acquisition function, particularly devised for real-valued boolean functions.

% Extension to categorical/ordinal variables
%In Section \ref{sec:algorithm}, we pointed out that the proposed algorithm can be extended to handle categorical and ordinal variables. In the future, we plan to investigate the performance of the proposed extension experimentally in combinatorial problems with variables of higher cardinality.

% Obtaining regret guarantees
%Finally, we note that none of the existing algorithms in the literature offer a regret guarantee on the performance of the relevant algorithm. We believe that the proposed algorithm is simple and intuitive enough to facilitate the derivation of such a regret guarantee. As such, obtaining a regret guarantee (on simple regret) for the proposed algorithm would be another interesting direction for further research.

%\begin{acks}
%Thanks to ...
%\end{acks}

\begin{acks}

We would like to thank Changyong Oh for helpful discussions on the COMBO algorithm.
Use of the Center for Nanoscale Materials, an Office of Science user facility, was supported by the U.S. Department of Energy, Office of Science, Office of Basic Energy Sciences, under Contract No. DE-AC02-06CH11357.

\end{acks}

% \clearpage

\bibliographystyle{ACM-Reference-Format}
\balance
%\bibliography{references.bib}

\begin{thebibliography}{40}

%%% the Web address.

\ifx \showCODEN    \undefined \def \showCODEN     #1{\unskip}     \fi
\ifx \showDOI      \undefined \def \showDOI       #1{#1}\fi
\ifx \showISBNx    \undefined \def \showISBNx     #1{\unskip}     \fi
\ifx \showISBNxiii \undefined \def \showISBNxiii  #1{\unskip}     \fi
\ifx \showISSN     \undefined \def \showISSN      #1{\unskip}     \fi
\ifx \showLCCN     \undefined \def \showLCCN      #1{\unskip}     \fi
\ifx \shownote     \undefined \def \shownote      #1{#1}          \fi
\ifx \showarticletitle \undefined \def \showarticletitle #1{#1}   \fi
\ifx \showURL      \undefined \def \showURL       {\relax}        \fi
% invisible to TeX
\providecommand\bibfield[2]{#2}
\providecommand\bibinfo[2]{#2}
\providecommand\natexlab[1]{#1}
\providecommand\showeprint[2][]{arXiv:#2}

\bibitem[\protect\citeauthoryear{Arora, Hazan, and Kale}{Arora
  et~al\mbox{.}}{2012}]%
        {arora2012multiplicative}
\bibfield{author}{\bibinfo{person}{Sanjeev Arora}, \bibinfo{person}{Elad
  Hazan}, {and} \bibinfo{person}{Satyen Kale}.}
  \bibinfo{year}{2012}\natexlab{}.
\newblock \showarticletitle{The multiplicative weights update method: a
  meta-algorithm and applications}.
\newblock \bibinfo{journal}{\emph{Theory of Computing}} \bibinfo{volume}{8},
  \bibinfo{number}{1} (\bibinfo{year}{2012}), \bibinfo{pages}{121--164}.
\newblock


\bibitem[\protect\citeauthoryear{Auer}{Auer}{2002}]%
        {auer2002using}
\bibfield{author}{\bibinfo{person}{Peter Auer}.}
  \bibinfo{year}{2002}\natexlab{}.
\newblock \showarticletitle{Using confidence bounds for
  exploitation-exploration trade-offs}.
\newblock \bibinfo{journal}{\emph{Journal of Machine Learning Research}}
  \bibinfo{volume}{3}, \bibinfo{number}{Nov} (\bibinfo{year}{2002}),
  \bibinfo{pages}{397--422}.
\newblock


\bibitem[\protect\citeauthoryear{Bell and Stevens}{Bell and Stevens}{2009}]%
        {Bell2009survey}
\bibfield{author}{\bibinfo{person}{Jordan Bell} {and} \bibinfo{person}{Brett
  Stevens}.} \bibinfo{year}{2009}\natexlab{}.
\newblock \showarticletitle{A survey of known results and research areas for
  n-queens}.
\newblock \bibinfo{journal}{\emph{Discrete Mathematics}} \bibinfo{volume}{309},
  \bibinfo{number}{1} (\bibinfo{year}{2009}), \bibinfo{pages}{1--31}.
\newblock


\bibitem[\protect\citeauthoryear{Bergstra and Bengio}{Bergstra and
  Bengio}{2012}]%
        {RS}
\bibfield{author}{\bibinfo{person}{James Bergstra} {and}
  \bibinfo{person}{Yoshua Bengio}.} \bibinfo{year}{2012}\natexlab{}.
\newblock \showarticletitle{Random Search for Hyper-Parameter Optimization.}
\newblock \bibinfo{journal}{\emph{J. Mach. Learn. Res.}}  \bibinfo{volume}{13}
  (\bibinfo{year}{2012}), \bibinfo{pages}{281--305}.
\newblock


\bibitem[\protect\citeauthoryear{Bergstra, Bardenet, Bengio, and
  K{\'e}gl}{Bergstra et~al\mbox{.}}{2011}]%
        {bergstra2011algorithms}
\bibfield{author}{\bibinfo{person}{James~S Bergstra}, \bibinfo{person}{R{\'e}mi
  Bardenet}, \bibinfo{person}{Yoshua Bengio}, {and} \bibinfo{person}{Bal{\'a}zs
  K{\'e}gl}.} \bibinfo{year}{2011}\natexlab{}.
\newblock \showarticletitle{Algorithms for hyper-parameter optimization}. In
  \bibinfo{booktitle}{\emph{Advances in neural information processing
  systems}}. \bibinfo{pages}{2546--2554}.
\newblock


\bibitem[\protect\citeauthoryear{Cesa-Bianchi and Lugosi}{Cesa-Bianchi and
  Lugosi}{2006}]%
        {prediction_ebook}
\bibfield{author}{\bibinfo{person}{Nicolò Cesa-Bianchi} {and}
  \bibinfo{person}{Gábor Lugosi}.} \bibinfo{year}{2006}\natexlab{}.
\newblock \bibinfo{booktitle}{\emph{Prediction, learning, and games.}}
\newblock \bibinfo{publisher}{Cambridge University Press}.
\newblock


\bibitem[\protect\citeauthoryear{Djolonga, Krause, and Cevher}{Djolonga
  et~al\mbox{.}}{2013}]%
        {djolonga2013high}
\bibfield{author}{\bibinfo{person}{Josip Djolonga}, \bibinfo{person}{Andreas
  Krause}, {and} \bibinfo{person}{Volkan Cevher}.}
  \bibinfo{year}{2013}\natexlab{}.
\newblock \showarticletitle{High-dimensional gaussian process bandits}. In
  \bibinfo{booktitle}{\emph{Advances in Neural Information Processing
  Systems}}. \bibinfo{pages}{1025--1033}.
\newblock


\bibitem[\protect\citeauthoryear{Gerchinovitz and Yu}{Gerchinovitz and
  Yu}{2011}]%
        {adaptive_EG}
\bibfield{author}{\bibinfo{person}{S{\'e}bastien Gerchinovitz} {and}
  \bibinfo{person}{Jia~Yuan Yu}.} \bibinfo{year}{2011}\natexlab{}.
\newblock \showarticletitle{Adaptive and Optimal Online Linear Regression on
  {$\ell1$}-Balls}. In \bibinfo{booktitle}{\emph{Algorithmic Learning Theory}}.
  \bibinfo{publisher}{Springer Berlin Heidelberg}, \bibinfo{pages}{99--113}.
\newblock


\bibitem[\protect\citeauthoryear{{Hardt} and {Rothblum}}{{Hardt} and
  {Rothblum}}{2010}]%
        {Moritz2010}
\bibfield{author}{\bibinfo{person}{M. {Hardt}} {and} \bibinfo{person}{G.~N.
  {Rothblum}}.} \bibinfo{year}{2010}\natexlab{}.
\newblock \showarticletitle{A Multiplicative Weights Mechanism for
  Privacy-Preserving Data Analysis}. In \bibinfo{booktitle}{\emph{2010 IEEE
  51st Annual Symposium on Foundations of Computer Science}}.
  \bibinfo{pages}{61--70}.
\newblock


\bibitem[\protect\citeauthoryear{Hebbal, Brevault, Balesdent, Talbi, and
  Melab}{Hebbal et~al\mbox{.}}{2019}]%
        {hebbal2019bayesian}
\bibfield{author}{\bibinfo{person}{Ali Hebbal}, \bibinfo{person}{Loic
  Brevault}, \bibinfo{person}{Mathieu Balesdent}, \bibinfo{person}{El-Ghazali
  Talbi}, {and} \bibinfo{person}{Nouredine Melab}.}
  \bibinfo{year}{2019}\natexlab{}.
\newblock \showarticletitle{Bayesian optimization using deep Gaussian
  processes}.
\newblock \bibinfo{journal}{\emph{arXiv preprint arXiv:1905.03350}}
  (\bibinfo{year}{2019}).
\newblock


\bibitem[\protect\citeauthoryear{Homaifar, Turner, and Ali}{Homaifar
  et~al\mbox{.}}{1992}]%
        {homaifar1992n}
\bibfield{author}{\bibinfo{person}{Abdollah Homaifar}, \bibinfo{person}{Joseph
  Turner}, {and} \bibinfo{person}{Samia Ali}.} \bibinfo{year}{1992}\natexlab{}.
\newblock \showarticletitle{The n-queens problem and genetic algorithms}. In
  \bibinfo{booktitle}{\emph{Proceedings IEEE Southeastcon'92}}. IEEE,
  \bibinfo{pages}{262--267}.
\newblock


\bibitem[\protect\citeauthoryear{Hu, Eberhart, and Shi}{Hu
  et~al\mbox{.}}{2003}]%
        {Hu2003swarm}
\bibfield{author}{\bibinfo{person}{Xiaohui Hu}, \bibinfo{person}{Russell~C
  Eberhart}, {and} \bibinfo{person}{Yuhui Shi}.}
  \bibinfo{year}{2003}\natexlab{}.
\newblock \showarticletitle{Swarm intelligence for permutation optimization: a
  case study of n-queens problem}. In \bibinfo{booktitle}{\emph{Proceedings of
  the 2003 IEEE Swarm Intelligence Symposium. SIS'03}}. IEEE,
  \bibinfo{pages}{243--246}.
\newblock


\bibitem[\protect\citeauthoryear{Hu, Hu, Xu, and Wang}{Hu
  et~al\mbox{.}}{2010}]%
        {contamination_control}
\bibfield{author}{\bibinfo{person}{Y. Hu}, \bibinfo{person}{J. Hu},
  \bibinfo{person}{Y. Xu}, {and} \bibinfo{person}{F. Wang}.}
  \bibinfo{year}{2010}\natexlab{}.
\newblock \showarticletitle{Contamination control in food supply chain}. In
  \bibinfo{booktitle}{\emph{Proceedings of the 2010 Winter Simulation
  Conference}}.
\newblock


\bibitem[\protect\citeauthoryear{Jones, Schonlau, and Welch}{Jones
  et~al\mbox{.}}{1998}]%
        {jones1998efficient}
\bibfield{author}{\bibinfo{person}{Donald~R Jones}, \bibinfo{person}{Matthias
  Schonlau}, {and} \bibinfo{person}{William~J Welch}.}
  \bibinfo{year}{1998}\natexlab{}.
\newblock \showarticletitle{Efficient global optimization of expensive
  black-box functions}.
\newblock \bibinfo{journal}{\emph{Journal of Global optimization}}
  \bibinfo{volume}{13}, \bibinfo{number}{4} (\bibinfo{year}{1998}),
  \bibinfo{pages}{455--492}.
\newblock


\bibitem[\protect\citeauthoryear{Kandasamy, Dasarathy, Schneider, and
  P{\'o}czos}{Kandasamy et~al\mbox{.}}{2017}]%
        {kandasamy2017multi}
\bibfield{author}{\bibinfo{person}{Kirthevasan Kandasamy},
  \bibinfo{person}{Gautam Dasarathy}, \bibinfo{person}{Jeff Schneider}, {and}
  \bibinfo{person}{Barnab{\'a}s P{\'o}czos}.} \bibinfo{year}{2017}\natexlab{}.
\newblock \showarticletitle{Multi-fidelity bayesian optimisation with
  continuous approximations}. In \bibinfo{booktitle}{\emph{International
  Conference on Machine Learning}}. \bibinfo{pages}{1799--1808}.
\newblock


\bibitem[\protect\citeauthoryear{Kirkpatrick, Gelatt, and Vecchi}{Kirkpatrick
  et~al\mbox{.}}{1983}]%
        {Kirkpatrick1983}
\bibfield{author}{\bibinfo{person}{S. Kirkpatrick}, \bibinfo{person}{C.~D.
  Gelatt}, {and} \bibinfo{person}{M.~P. Vecchi}.}
  \bibinfo{year}{1983}\natexlab{}.
\newblock \showarticletitle{Optimization by Simulated Annealing}.
\newblock \bibinfo{journal}{\emph{Science}} \bibinfo{volume}{220},
  \bibinfo{number}{4598} (\bibinfo{year}{1983}), \bibinfo{pages}{671--680}.
\newblock


\bibitem[\protect\citeauthoryear{Kivinen and Warmuth}{Kivinen and
  Warmuth}{1997}]%
        {KIVINEN1997}
\bibfield{author}{\bibinfo{person}{Jyrki Kivinen} {and}
  \bibinfo{person}{Manfred~K. Warmuth}.} \bibinfo{year}{1997}\natexlab{}.
\newblock \showarticletitle{Exponentiated Gradient versus Gradient Descent for
  Linear Predictors}.
\newblock \bibinfo{journal}{\emph{Information and Computation}}
  \bibinfo{volume}{132}, \bibinfo{number}{1} (\bibinfo{year}{1997}),
  \bibinfo{pages}{1 -- 63}.
\newblock


\bibitem[\protect\citeauthoryear{Li, Jamieson, DeSalvo, Rostamizadeh, and
  Talwalkar}{Li et~al\mbox{.}}{2017}]%
        {li2017hyperband}
\bibfield{author}{\bibinfo{person}{Lisha Li}, \bibinfo{person}{Kevin Jamieson},
  \bibinfo{person}{Giulia DeSalvo}, \bibinfo{person}{Afshin Rostamizadeh},
  {and} \bibinfo{person}{Ameet Talwalkar}.} \bibinfo{year}{2017}\natexlab{}.
\newblock \showarticletitle{Hyperband: A novel bandit-based approach to
  hyperparameter optimization}.
\newblock \bibinfo{journal}{\emph{The Journal of Machine Learning Research}}
  \bibinfo{volume}{18}, \bibinfo{number}{1} (\bibinfo{year}{2017}),
  \bibinfo{pages}{6765--6816}.
\newblock


\bibitem[\protect\citeauthoryear{Mo{\v{c}}kus}{Mo{\v{c}}kus}{1975}]%
        {movckus1975bayesian}
\bibfield{author}{\bibinfo{person}{Jonas Mo{\v{c}}kus}.}
  \bibinfo{year}{1975}\natexlab{}.
\newblock \showarticletitle{On Bayesian methods for seeking the extremum}. In
  \bibinfo{booktitle}{\emph{Optimization techniques IFIP technical
  conference}}. Springer, \bibinfo{pages}{400--404}.
\newblock


\bibitem[\protect\citeauthoryear{Mockus}{Mockus}{1994}]%
        {mockus1994application}
\bibfield{author}{\bibinfo{person}{Jonas Mockus}.}
  \bibinfo{year}{1994}\natexlab{}.
\newblock \showarticletitle{Application of Bayesian approach to numerical
  methods of global and stochastic optimization}.
\newblock \bibinfo{journal}{\emph{Journal of Global Optimization}}
  \bibinfo{volume}{4}, \bibinfo{number}{4} (\bibinfo{year}{1994}),
  \bibinfo{pages}{347--365}.
\newblock


\bibitem[\protect\citeauthoryear{Mukherjee, Datta, Chanda, and
  Pathak}{Mukherjee et~al\mbox{.}}{2015}]%
        {Mukherjee2015}
\bibfield{author}{\bibinfo{person}{Soham Mukherjee}, \bibinfo{person}{Santanu
  Datta}, \bibinfo{person}{Pramit~Brata Chanda}, {and} \bibinfo{person}{Pratik
  Pathak}.} \bibinfo{year}{2015}\natexlab{}.
\newblock \showarticletitle{Comparative Study of Different Algorithms to Solve
  N-Queens Problem}.
\newblock \bibinfo{journal}{\emph{International Journal of Foundations of
  Computer Science and Technology}} \bibinfo{volume}{5}, \bibinfo{number}{2}
  (\bibinfo{year}{2015}), \bibinfo{pages}{15--27}.
\newblock


\bibitem[\protect\citeauthoryear{O'Donnell}{O'Donnell}{2014}]%
        {Boolean}
\bibfield{author}{\bibinfo{person}{Ryan O'Donnell}.}
  \bibinfo{year}{2014}\natexlab{}.
\newblock \bibinfo{booktitle}{\emph{Analysis of Boolean Functions}}.
\newblock \bibinfo{publisher}{Cambridge University Press}.
\newblock


\bibitem[\protect\citeauthoryear{Oh, Tomczak, Gavves, and Welling}{Oh
  et~al\mbox{.}}{2019}]%
        {COMBO}
\bibfield{author}{\bibinfo{person}{Changyong Oh}, \bibinfo{person}{Jakub
  Tomczak}, \bibinfo{person}{Efstratios Gavves}, {and} \bibinfo{person}{Max
  Welling}.} \bibinfo{year}{2019}\natexlab{}.
\newblock \showarticletitle{Combinatorial {Bayesian} Optimization using the
  Graph Cartesian Product}.
\newblock In \bibinfo{booktitle}{\emph{Advances in Neural Information
  Processing Systems 32}}. \bibinfo{pages}{2910--2920}.
\newblock


\bibitem[\protect\citeauthoryear{Ostadhossein, Rahnamoun, Wang, Zhao, Zhang,
  Crespi, and van Duin}{Ostadhossein et~al\mbox{.}}{2017}]%
        {ostadhossein2017reaxff}
\bibfield{author}{\bibinfo{person}{Alireza Ostadhossein}, \bibinfo{person}{Ali
  Rahnamoun}, \bibinfo{person}{Yuanxi Wang}, \bibinfo{person}{Peng Zhao},
  \bibinfo{person}{Sulin Zhang}, \bibinfo{person}{Vincent~H Crespi}, {and}
  \bibinfo{person}{Adri~CT van Duin}.} \bibinfo{year}{2017}\natexlab{}.
\newblock \showarticletitle{ReaxFF reactive force-field study of molybdenum
  disulfide (MoS2)}.
\newblock \bibinfo{journal}{\emph{The journal of physical chemistry letters}}
  \bibinfo{volume}{8}, \bibinfo{number}{3} (\bibinfo{year}{2017}),
  \bibinfo{pages}{631--640}.
\newblock


\bibitem[\protect\citeauthoryear{Patra, Zhang, Schulman, Chan, Cherukara,
  Terrones, Das, Narayanan, and Sankaranarayanan}{Patra et~al\mbox{.}}{2018}]%
        {defect_dynamics}
\bibfield{author}{\bibinfo{person}{Tarak~K. Patra}, \bibinfo{person}{Fu Zhang},
  \bibinfo{person}{Daniel~S. Schulman}, \bibinfo{person}{Henry Chan},
  \bibinfo{person}{Mathew~J. Cherukara}, \bibinfo{person}{Mauricio Terrones},
  \bibinfo{person}{Saptarshi Das}, \bibinfo{person}{Badri Narayanan}, {and}
  \bibinfo{person}{Subramanian K. R.~S. Sankaranarayanan}.}
  \bibinfo{year}{2018}\natexlab{}.
\newblock \showarticletitle{Defect Dynamics in 2-D MoS2 Probed by Using Machine
  Learning, Atomistic Simulations, and High-Resolution Microscopy}.
\newblock \bibinfo{journal}{\emph{ACS Nano}} \bibinfo{volume}{12},
  \bibinfo{number}{8} (\bibinfo{year}{2018}), \bibinfo{pages}{8006--8016}.
\newblock


\bibitem[\protect\citeauthoryear{Plimpton}{Plimpton}{1995}]%
        {Plimpton1995lammps}
\bibfield{author}{\bibinfo{person}{Steve Plimpton}.}
  \bibinfo{year}{1995}\natexlab{}.
\newblock \showarticletitle{Fast Parallel Algorithms for Short-Range Molecular
  Dynamics}.
\newblock \bibinfo{journal}{\emph{J. Comput. Phys.}} \bibinfo{volume}{117},
  \bibinfo{number}{1} (\bibinfo{year}{1995}), \bibinfo{pages}{1--19}.
\newblock


\bibitem[\protect\citeauthoryear{Ricardo~Baptista}{Ricardo~Baptista}{2018}]%
        {BOCS}
\bibfield{author}{\bibinfo{person}{Matthias~Poloczek Ricardo~Baptista}.}
  \bibinfo{year}{2018}\natexlab{}.
\newblock \showarticletitle{Bayesian Optimization of Combinatorial Structures}.
  In \bibinfo{booktitle}{\emph{International Conference on International
  Conference on Machine Learning (ICML)}}.
\newblock


\bibitem[\protect\citeauthoryear{Sch\"{a}fer}{Sch\"{a}fer}{2013}]%
        {Schafer2013}
\bibfield{author}{\bibinfo{person}{Christian Sch\"{a}fer}.}
  \bibinfo{year}{2013}\natexlab{}.
\newblock \showarticletitle{Particle Algorithms for Optimization on Binary
  Spaces}.
\newblock \bibinfo{journal}{\emph{ACM Transactions on Modeling and Computer
  Simulation (TOMACS)}} \bibinfo{volume}{23}, \bibinfo{number}{1}
  (\bibinfo{year}{2013}).
\newblock


\bibitem[\protect\citeauthoryear{Sen, Kandasamy, and Shakkottai}{Sen
  et~al\mbox{.}}{2018}]%
        {sen2018multi}
\bibfield{author}{\bibinfo{person}{Rajat Sen}, \bibinfo{person}{Kirthevasan
  Kandasamy}, {and} \bibinfo{person}{Sanjay Shakkottai}.}
  \bibinfo{year}{2018}\natexlab{}.
\newblock \showarticletitle{Multi-fidelity black-box optimization with
  hierarchical partitions}. In \bibinfo{booktitle}{\emph{International
  conference on machine learning}}. \bibinfo{pages}{4538--4547}.
\newblock


\bibitem[\protect\citeauthoryear{{Shahriari}, {Swersky}, {Wang}, {Adams}, and
  {de Freitas}}{{Shahriari} et~al\mbox{.}}{2016}]%
        {BO}
\bibfield{author}{\bibinfo{person}{B. {Shahriari}}, \bibinfo{person}{K.
  {Swersky}}, \bibinfo{person}{Z. {Wang}}, \bibinfo{person}{R.~P. {Adams}},
  {and} \bibinfo{person}{N. {de Freitas}}.} \bibinfo{year}{2016}\natexlab{}.
\newblock \showarticletitle{Taking the Human Out of the Loop: A Review of
  Bayesian Optimization}.
\newblock \bibinfo{journal}{\emph{Proc. IEEE}} (\bibinfo{year}{2016}).
\newblock


\bibitem[\protect\citeauthoryear{Spears}{Spears}{1993}]%
        {SA}
\bibfield{author}{\bibinfo{person}{William~M. Spears}.}
  \bibinfo{year}{1993}\natexlab{}.
\newblock \showarticletitle{Simulated Annealing for Hard Satisfiability
  Problems}. In \bibinfo{booktitle}{\emph{DIMACS Workshop: Cliques, Coloring,
  and Satisfiability}}. \bibinfo{pages}{533--557}.
\newblock


\bibitem[\protect\citeauthoryear{Srinivas, Krause, Kakade, and Seeger}{Srinivas
  et~al\mbox{.}}{2010}]%
        {srinivas2009gaussian}
\bibfield{author}{\bibinfo{person}{Niranjan Srinivas}, \bibinfo{person}{Andreas
  Krause}, \bibinfo{person}{Sham Kakade}, {and} \bibinfo{person}{Matthias
  Seeger}.} \bibinfo{year}{2010}\natexlab{}.
\newblock \showarticletitle{Gaussian Process Optimization in the Bandit
  Setting: No Regret and Experimental Design}. In
  \bibinfo{booktitle}{\emph{Proceedings of the 27th International Conference on
  International Conference on Machine Learning}}
  \emph{(\bibinfo{series}{ICML’10})}. \bibinfo{pages}{1015–1022}.
\newblock


\bibitem[\protect\citeauthoryear{Takenaka, Funabiki, and Higashino}{Takenaka
  et~al\mbox{.}}{2000}]%
        {takenaka2000proposal}
\bibfield{author}{\bibinfo{person}{Yoichi Takenaka}, \bibinfo{person}{Nobuo
  Funabiki}, {and} \bibinfo{person}{Teruo Higashino}.}
  \bibinfo{year}{2000}\natexlab{}.
\newblock \showarticletitle{A proposal of neuron filter: A constraint
  resolution scheme of neural networks for combinatorial optimization
  problems}.
\newblock \bibinfo{journal}{\emph{IEICE Transactions on Fundamentals of
  Electronics, Communications and Computer Sciences}} \bibinfo{volume}{83},
  \bibinfo{number}{9} (\bibinfo{year}{2000}), \bibinfo{pages}{1815--1823}.
\newblock


\bibitem[\protect\citeauthoryear{Thompson}{Thompson}{1933}]%
        {thompson1933likelihood}
\bibfield{author}{\bibinfo{person}{William~R Thompson}.}
  \bibinfo{year}{1933}\natexlab{}.
\newblock \showarticletitle{On the likelihood that one unknown probability
  exceeds another in view of the evidence of two samples}.
\newblock \bibinfo{journal}{\emph{Biometrika}} \bibinfo{volume}{25},
  \bibinfo{number}{3/4} (\bibinfo{year}{1933}), \bibinfo{pages}{285--294}.
\newblock


\bibitem[\protect\citeauthoryear{van~der Hoeven, van Erven, and
  Kot{\l}owski}{van~der Hoeven et~al\mbox{.}}{2018}]%
        {hoeven18}
\bibfield{author}{\bibinfo{person}{Dirk van~der Hoeven}, \bibinfo{person}{Tim
  van Erven}, {and} \bibinfo{person}{Wojciech Kot{\l}owski}.}
  \bibinfo{year}{2018}\natexlab{}.
\newblock \showarticletitle{The Many Faces of Exponential Weights in Online
  Learning}. In \bibinfo{booktitle}{\emph{Proceedings of the 31st Conference On
  Learning Theory}}, Vol.~\bibinfo{volume}{75}. \bibinfo{pages}{2067--2092}.
\newblock


\bibitem[\protect\citeauthoryear{Vovk}{Vovk}{1998}]%
        {Vovk1998}
\bibfield{author}{\bibinfo{person}{V Vovk}.} \bibinfo{year}{1998}\natexlab{}.
\newblock \showarticletitle{A Game of Prediction with Expert Advice}.
\newblock \bibinfo{journal}{\emph{J. Comput. Syst. Sci.}} \bibinfo{volume}{56},
  \bibinfo{number}{2} (\bibinfo{year}{1998}), \bibinfo{pages}{153–173}.
\newblock


\bibitem[\protect\citeauthoryear{Wang, Zoghi, Hutter, Matheson, and
  De~Freitas}{Wang et~al\mbox{.}}{2013}]%
        {wang2013bayesian}
\bibfield{author}{\bibinfo{person}{Ziyu Wang}, \bibinfo{person}{Masrour Zoghi},
  \bibinfo{person}{Frank Hutter}, \bibinfo{person}{David Matheson}, {and}
  \bibinfo{person}{Nando De~Freitas}.} \bibinfo{year}{2013}\natexlab{}.
\newblock \showarticletitle{Bayesian optimization in high dimensions via random
  embeddings}. In \bibinfo{booktitle}{\emph{Twenty-Third International Joint
  Conference on Artificial Intelligence}}.
\newblock


\bibitem[\protect\citeauthoryear{Williamson and Shmoys}{Williamson and
  Shmoys}{2011}]%
        {williamson2011design}
\bibfield{author}{\bibinfo{person}{David~P Williamson} {and}
  \bibinfo{person}{David~B Shmoys}.} \bibinfo{year}{2011}\natexlab{}.
\newblock \bibinfo{booktitle}{\emph{The design of approximation algorithms}}.
\newblock \bibinfo{publisher}{Cambridge university press}.
\newblock


\bibitem[\protect\citeauthoryear{Wolsey and Nemhauser}{Wolsey and
  Nemhauser}{1999}]%
        {wolsey1999integer}
\bibfield{author}{\bibinfo{person}{Laurence~A Wolsey} {and}
  \bibinfo{person}{George~L Nemhauser}.} \bibinfo{year}{1999}\natexlab{}.
\newblock \bibinfo{booktitle}{\emph{Integer and combinatorial optimization}}.
  Vol.~\bibinfo{volume}{55}.
\newblock \bibinfo{publisher}{John Wiley \& Sons}.
\newblock


\bibitem[\protect\citeauthoryear{Yoav~Freund}{Yoav~Freund}{1997}]%
        {hedge}
\bibfield{author}{\bibinfo{person}{Robert E.~Schapire Yoav~Freund}.}
  \bibinfo{year}{1997}\natexlab{}.
\newblock \showarticletitle{A Decision-Theoretic Generalization of On-Line
  Learning and an Application to Boosting}.
\newblock \bibinfo{journal}{\emph{J. Comput. System Sci.}}
  (\bibinfo{year}{1997}).
\newblock


\end{thebibliography}

\clearpage

\appendix

\section{Extension of Lemma \ref{lemma_positive}}
\label{app:lemma_extension}

The results of Lemma \ref{lemma_positive} can be extended to the general case of Fourier coefficients with arbitrary signs, as follows.

\begin{remark}
\label{remark_general}
Lemma \ref{lemma_positive} holds for the general case of Fourier coefficients $\alpha_i^t$ with arbitrary signs.
\end{remark}
\begin{proof}
Following on the idea from \cite{hoeven18}, if $\lVert \alpha^{*} \rVert_1 \leq 1$, then one can always write $\alpha^{*}_i = \alpha^{*}_{i,+}-\alpha^{*}_{i,-}$ where $\sum_{\gamma,i}\alpha^{*}_{i,\gamma}=1$ and $\alpha^{*}_{i,\gamma} \geq 0$. This is because any point inside an $\ell_1$ ball is in the convex hull of $\{e_i,-e_i\}_{i \in [d]}$ where $e_i$ are the canonical unit vectors. Therefore, to approximate it at any time $t$ during the algorithm with exponential weight update, we assume that we have a set of $2p$ Fourier coefficients; we consider the monomial terms $+\psi_i(x)$ for the Fourier coefficients $\alpha_{i, +}^t$ as well as the monomial terms $-\psi_i(x)$ for the Fourier coefficients $\alpha_{i, -}^t$. Note that all the coefficients $\alpha_{i, \gamma}^t, \forall \gamma \in \{-, +\}$ are non-negative, and that the set of all such coefficients form a distribution, i.e. $\sum_{i, \gamma} \alpha_{i, \gamma}^t = 1$, due to the normalization in Algorithm \ref{algo:COMEX}.  Therefore, applying Lemma \ref{lemma_positive} to the extended set of Fourier coefficients completes the proof.
\end{proof}

\section{Proof of Theorem \ref{thm:Acquisition}}

In the main paper, we used a simulated annealing based acquisition function on the surrogate model $\hat{f}(x)$. This model is very difficult to analyze. Instead, we analyze a more idealized and simpler acquisition function which is 
$$x \sim \frac{\exp\left(-\nicefrac{\hat{f}_{\alpha_t}(x)}{T}\right)}{\sum_{x\in \{-1,1\}^d} \exp \left(-\nicefrac{\hat{f}_{\alpha_t}(x)}{T}\right)}$$ 
where $T$ is the temperature. Let the p.m.f of this acquired sample be $\hat{P}_{\alpha_t}(x)$.

Similarly, if we had access to the actual $f$, we would be using the acquisition function: 
$$x \sim \frac{\exp\left(-\nicefrac{f(x)}{T}\right)}{\sum_{x\in \{-1,1\}^d} \exp \left(-\nicefrac{f(x)}{T}\right)}.$$
Let the p.m.f of this acquired sample be $P(x)$.

Now, we show a result that implies the following under some additional technical condition: \textit{Until the acquisition function based on $\hat{f}_{\alpha_t}$ yield samples which is close in KL divergence to samples yielded by acquisition function based on $f$, average $\phi_{t-1}- \phi_{t}$ (as in Lemma \ref{lemma_positive}) is large.}

Let $\hat{Z}= \sum_{x} \exp \left( - \hat{f}_{\alpha_t}(x)/T \right)$ be the partition function. Similarly, let $Z$ be the partition function associated with $P(x)$.

\begin{theorem}[Theorem \ref{thm:Acquisition} restated]
 Let $$-1 \leq \hat{f}_{\alpha_t}(x), f(x) \leq 1.$$ If at any time $t$, and for a temperature $T$, we have for some $\epsilon>0$: \[ \bigg \lvert D_{\mathrm{KL}}(\hat{P}_{\alpha_t} \lVert P) -  \log \left ( \frac{ Z}{\hat{Z}} \right ) \bigg \rvert \geq \epsilon ,\]
 then $\mathbb{E}_{\hat{P}_{\alpha_t}} [\phi_{t-1}-\phi_t] \geq 2 \eta \lambda \epsilon^2 T^2 - \eta^2$. Here, $D_{\mathrm{KL}}$ is defined with respect to $\log_e$ for convenience.
\end{theorem}
\begin{proof}

%Observe that:
%\begin{align}\label{ineq:partition}
%\log \frac{Z}{\hat{Z}} \leq %\frac{2}{T}.
%\end{align}
% This is because  $ 2^d \times  (\exp^{-1/T})  \leq \hat{Z},Z \leq  2^d \times  (\exp^{1/T}) $.

We have:
 \begin{align}
    \epsilon \leq \bigg \lvert D_{\mathrm{KL}}(\hat{P}_{\alpha_t} \lVert P) - %\frac{2}{T}&=
    \log \left(\frac{Z}{\hat{Z}} \right)  \bigg \rvert     &=  \bigg  \lvert \mathbb{E}_{\hat{P}_{\alpha_t}} [ -\frac{1}{T} (\hat{f}_{\alpha_t}(x)- f(x))]  \nonumber \\
    & + \log \left ( \frac{ Z}{\hat{Z}} \right ) - \log \left ( \frac{ Z}{\hat{Z}} \right ) \bigg \rvert \nonumber \\
    & = \bigg \lvert \mathbb{E}_{\hat{P}_{\alpha_t}} [ -\frac{1}{T} (\hat{f}_{\alpha_t}(x)- f(x))] \bigg \rvert \nonumber \\
    &\overset{a}{\leq} \frac{1}{T}\mathbb{E}_{\hat{P}_{\alpha_t}} [ \lvert  \hat{f}_{\alpha_t}(x)- f(x) \rvert] \nonumber \\
    & \overset{b}{\leq} \frac{1}{T}\sqrt{\mathbb{E}_{\hat{P}_{\alpha_t}} [ \lvert  \hat{f}_{\alpha_t}(x)- f(x) \rvert^2 ]}
 \end{align}
 Justifications: (a) Jensen's inequality applied to  $|x|$. (b) Jensen's inequality applied to the function $x^2$, i.e  $(\mathbb{E}[\lvert X \rvert])^2 \leq \mathbb{E}[ X^2 ] $.

Combined with Lemma \ref{lemma_positive}, this implies:
\begin{align}
\mathbb{E}_{\hat{P}_{\alpha_t}} [\phi_{t-1}-\phi_t] \geq 2 \eta \lambda \epsilon^2   T^2 - \eta^2.
\end{align}

\end{proof}
%\textbf{Remark:} We notice that in our schedule for simulated annealing we set $T=\exp(-t/d)$ in Algorithm \ref{algo:SA}.

\section{Learning Rate in Algorithm \ref{algo:COMEX}}
\label{app:learning_rate}

In Algorithm \ref{algo:COMEX}, we use the anytime learning rate schedule of \cite{adaptive_EG}, which is a decreasing function of time $t$. The learning rate at time step $t$ is given by:
\begin{equation}
	\eta_t = \min \bigg \{ \frac{1}{e_{t-1}}, c \sqrt{\frac{\ln{(2 \, p)}}{v_{t-1}}} \bigg \},
\end{equation}
where $c \overset{\Delta}{=} \sqrt{2(\sqrt{2} - 1)/(\exp(1)-2)}$ and 
\begin{align*}
    z_{j, t}^{\gamma} &\overset{\Delta}{=} - 2 \, \gamma \, \lambda \, \ell_t \, \psi_j(x_t) \\
    e_t &\overset{\Delta}{=} \inf_{k \in \mathbb{Z}} \bigg \{ 2^k: 2^k \geq \max_{s \in [t]} \max_{ \substack{j, k \in [p] \\ \gamma, \mu \in \{-, +\}}} | z_{j, s}^{\gamma} - z_{k, s}^{\mu} |  \bigg \} \\
    v_t &\overset{\Delta}{=}  \sum_{s \in [t]} \sum_{\substack{j \in [p] \\ \gamma \in \{-, +\}} } \alpha_{j, s}^{\gamma} \bigg ( z_{j, s}^{\gamma} - \sum_{\substack{k \in [p] \\ \mu \in \{-, +\}}} \alpha_{k, s}^{\mu} z_{k, s}^{\mu} \bigg )^2.
\end{align*}

\end{document}